\documentclass[iicol,sn-mathphys-num]{sn-jnl}

\usepackage{graphicx}%
\usepackage{multirow}%
\usepackage{amsmath,amssymb,amsfonts}%
\usepackage{amsthm}%
\usepackage{mathrsfs}%
\usepackage[title]{appendix}%
\usepackage{xcolor}%
\usepackage{textcomp}%
\usepackage{stfloats}
\usepackage{balance}
\usepackage{url}
\usepackage{verbatim}
\usepackage{manyfoot}%
\usepackage{booktabs}%
\usepackage{algorithm}%
\usepackage{algorithmicx}%
\usepackage{algpseudocode}%
\usepackage{listings}%
\usepackage{arydshln} 
\usepackage{booktabs} 
\usepackage{colortbl} 
\usepackage{xcolor} 
\usepackage{hyperref}
\hypersetup{
    colorlinks=true,
    linkcolor=black,
    urlcolor=blue,
}

\usepackage{thmtools}
\declaretheoremstyle[%
  spaceabove=3pt,%
  spacebelow=3pt,%
  headfont=\normalfont\bfseries,%
  bodyfont=\normalfont\itshape,%
  postheadspace=0.5em,%
]{theoremstyle} 

\declaretheorem[name={Definition},style=theoremstyle]{definition}

\declaretheorem[name={Theorem},style=theoremstyle]{theorem}
\declaretheorem[name={Corollary},style=theoremstyle]{corollary}

\begin{document}


\title[CAMs as Shapley Value-based Explainers]
{CAMs as Shapley Value-based Explainers}







\author[1,2]{\fnm{Huaiguang} \sur{Cai}}\email{caihuaiguang2022@ia.ac.cn}

\affil[1]{\orgdiv{State Key Laboratory of Multimodal Artificial Intelligence Systems}, \orgname{Institute of Automation, Chinese Academy of Sciences}, \orgaddress{ \city{Beijing}, \country{China}}}

\affil[2]{\orgdiv{School of Artificial Intelligence}, \orgname{University of Chinese Academy of Sciences}, \orgaddress{ \city{Beijing}, \country{China}}}

\abstract{
Class Activation Mapping (CAM) methods are widely used to visualize neural network decisions, yet their underlying mechanisms remain incompletely understood. To enhance the understanding of CAM methods and improve their explainability, we introduce the Content Reserved Game-theoretic (CRG) Explainer. This theoretical framework clarifies the theoretical foundations of GradCAM and HiResCAM by modeling the neural network prediction process as a cooperative game. Within this framework, we develop ShapleyCAM, a new method that leverages gradients and the Hessian matrix to provide more precise and theoretically grounded visual explanations. Due to the computational infeasibility of exact Shapley value calculation, ShapleyCAM employs a second-order Taylor expansion of the cooperative game's utility function to derive a closed-form expression. Additionally, we propose the Residual Softmax Target-Class (ReST) utility function to address the limitations of pre-softmax and post-softmax scores. Extensive experiments across 12 popular networks on the ImageNet validation set demonstrate the effectiveness of ShapleyCAM and its variants. Our findings not only advance CAM explainability but also bridge the gap between heuristic-driven CAM methods and compute-intensive Shapley value-based methods. The code is available at \url{https://github.com/caihuaiguang/pytorch-shapley-cam}.
}
\keywords{explainable AI, class activation mapping (CAM), Shapley value, decision-making, game theory.}

\maketitle

\section{Introduction}
With the increasing reliance on machine learning models in critical fields such as healthcare diagnostics \cite{Julia2020EAIhealth} and autonomous driving \cite{daniel2022EAIautonomous}, the need for explainable AI (XAI) has never been more pressing. As these models are deployed in environments where human lives and safety are at stake, it becomes essential to understand the mechanism behind their predictions. To ensure these models are reliable and transparent, it is crucial to interpret their predictions as a decision-making process. This perspective is motivated by a basic intuition: for accurate predictions, knowing which features the network primarily relied upon helps ensure its behavior aligns with human logic \cite{zhou2016learning}. In cases of incorrect predictions, identifying the features driving the error can aid in detecting biases and debugging the model \cite{hires2020rachel,selvaraju2017grad}.

To achieve such explainability, CAM methods \cite{zhou2016learning, hires2020rachel, selvaraju2017grad, wang2020score, abalation2020desai, chattopadhay2018grad, fu2020axiom, jiang2021layercam, liftcam2021jung, jacobgilpytorchcam} have gained prominence. These methods generate visual explanations by identifying which regions of the input image most influence a model’s output. However, we find many methods, including GradCAM++ \cite{chattopadhay2018grad}, LayerCAM \cite{jiang2021layercam}, and GradCAM-E \cite{jacobgilpytorchcam}, often confuse localization ability with true explainability—a critical distinction that is frequently overlooked, as shown in Fig. \ref{fig:main}. More importantly, CAM methods frequently rely heavily on heuristics, and the absence of a solid theoretical foundation has become a major obstacle to their further development.


Meanwhile, the Shapley value \cite{shapley1953value} from cooperative game theory offers a well-established theoretical framework for fairly quantifying feature attribution \cite{roz2022shapleysurvey}. The Shapley value has been successfully applied in XAI like SHAP \cite{LundbergL17SHAP} and Data Shapley \cite{zou2019datashapley}. SHAP attributes the model's inference result to each input feature, and Data Shapley attributes the model's training performance to each training data point. However, the exponential complexity of computing exact Shapley values presents a significant practical challenge, limiting their scalability for high-dimensional input data or large datasets \cite{wang2024data}.

To enhance the understanding and applicability of CAM methods, we introduce the CRG Explainer. This theoretical framework marries the great scalability of CAM methods with the solid theoretical underpinning of the Shapley value. Within the CRG Explainer, we propose ShapleyCAM, a novel Shapley value-based CAM method designed to offer improved explainability of the decision-making process of neural networks. By bridging the gap between heuristic-driven CAM methods and compute-intensive Shapley value-based methods, ShapleyCAM ensures both scalability and fairness in feature attribution. The key contributions of this paper include:

\begin{itemize} 
\item \textbf{Content Reserved Game-theoretic Explainer:} This theoretical framework generalizes CAM methods by incorporating Shapley value, clarifying the theoretical basis of GradCAM (satisfying the content reserved property) and HiResCAM (satisfying the game-theoretic property), thereby establishing a connection between CAM methods and Shapley value-based methods.
\item \textbf{ShapleyCAM algorithm:} Within the CRG Explainer and based on a second-order approximation of the utility function, we develop ShapleyCAM, leveraging the gradient and Hessian matrix of neural networks to generate more accurate explanations. 
\item \textbf{ReST utility function:} We analyze the advantages and limitations of using pre and post-softmax scores, elucidating their theoretical relationship, and introduce the ReST utility function to overcome these limitations.
\item \textbf{Extensive empirical validation:} We conduct comprehensive experiments across 12 network architectures and 6 metrics on the ImageNet validation set, providing a thorough comparison of existing gradient-based CAM methods with ShapleyCAM and its variants, and demonstrating that incorporating both the gradient and Hessian matrix typically results in more accurate explanations.
\end{itemize}

\section{Related Work}

In this section, we introduce two representative types of methods in feature attribution: CAM methods and Shapley value-based methods. CAM methods generate heatmaps to show which regions of an image most influence the model's predictions, while Shapley value-based methods quantify the precise numerical contributions of individual features.


\begin{figure*}[htbp]
    \centering
    \includegraphics[width=\textwidth]{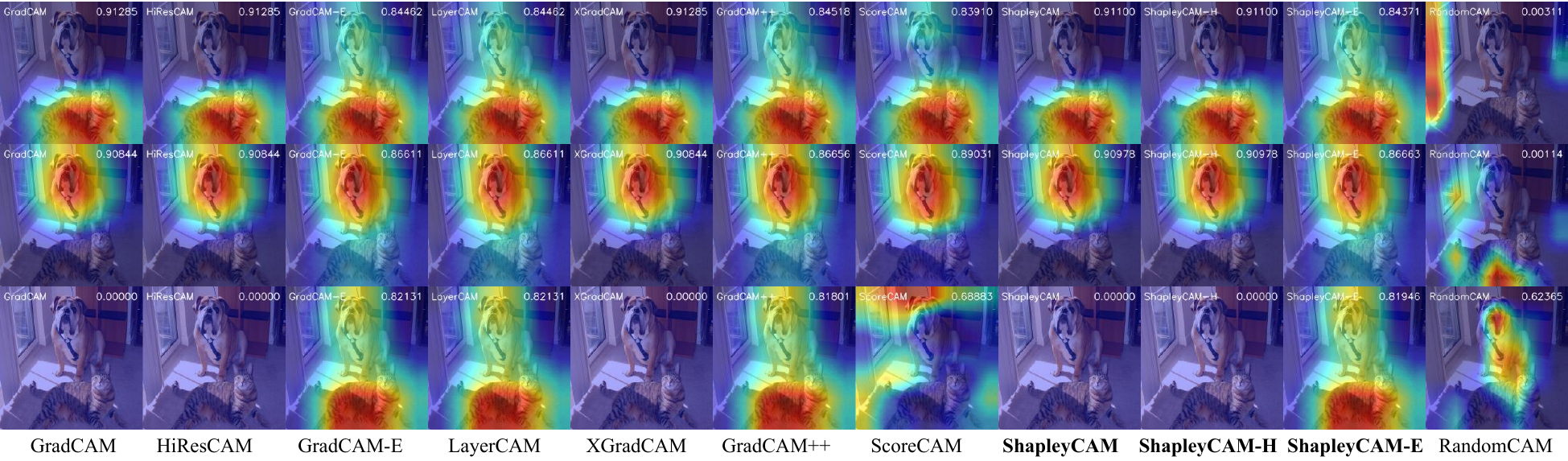} 
\caption{Visual explanation on ResNet-18 generated by various CAM methods using the ReST, with the layer preceding GAP as the target layer. The target classes from top to bottom are: \textit{tiger cat}, \textit{boxer}, and \textit{yellow lady's slipper} (the least likely class). A good explainer should avoid focusing on \textit{tiger cat} or \textit{boxer} when \textit{yellow lady's slipper} is the target class. Higher ADCC scores (top-right) indicate better performance.}
    \label{fig:main}
\end{figure*}
\subsection{CAM Methods}

CAM was proposed as a weakly supervised object localization approach based on the discovery that the global average pooling (GAP) layer actually exposes the implicit attention of convolutional neural networks (CNNs) on an image~\cite{zhou2016learning}.  Then, CAM also serves as a foundational explainable technique for visualizing CNN decision-making processes by highlighting the image regions that are critical to the model's predictions.

Specifically, after a well-trained neural network makes a class prediction on an image \(x\), CAM generates a heatmap \(\text{CAM}_c(x) \) for the target class \(c\). This heatmap is produced by linearly combining the activation maps  \(\{A^i\}_{i=1}^{N_l}\) of the target layer \(l\), typically the layer preceding GAP or the final convolutional layer where $N_l$ denotes the number of channels of layer \(l\), with the linear coefficients \(\{w^i\}_{i=1}^{N_l}\) being the weights associated with the target class in the fully connected layer following GAP. More formally:
\begin{equation}
    \text{CAM}_c(x) =  \sum_{i=1}^{N_l} w^i A^i.
\end{equation}
The heatmap is then normalized, upsampled to the original image size, and overlaid onto the original image to create a visual explanation (similar to Fig. \ref{fig:main}) of the model prediction. This post-processing step has become a standard operation in subsequent works and will therefore be omitted in the following. Notably, the fully connected layer outputs logits \(\{y^c\}_{c=1}^C\) (i.e., pre-softmax scores) for each class, which are typically transformed into probabilities \(\{p^c\}_{c=1}^C\) (i.e., post-softmax scores) via softmax function.

The original CAM is limited to architectures where GAP is followed by a fully connected layer functioning as a classifier \cite{minh2023overview}. GradCAM~\cite{selvaraju2017grad} generalizes CAM to any CNN architecture by using the gradient \( W^i = \frac{\partial y^c}{\partial A^i} \) of the pre-softmax score \( y^c \) with respect to the activation map \( A^i \) to compute the importance of the \( i \)-th activation map. Note that, GradCAM typically uses the output logit \(y^c\) to generate its heatmap. The gradients are averaged to weight the activation maps, and a ReLU operation is introduced to remove the negative regions, producing the final heatmap. Denoting \(\overline{W^i}\) as the mean of \( W^i \), and the heatmap is generated as follows:
\begin{equation}\label{eq:gradcam}
    \text{GradCAM}_c(x) =  \text{ReLU}\left(\sum_{i=1}^{N_l} \overline{W^i} A^i\right).
\end{equation}
HiResCAM~\cite{hires2020rachel} provides a more faithful and finer-grain explanation by element-wise multiplying the activations with the gradients, where \(\odot\) refers to Hadamard product:
\begin{equation}\label{eq:hirescam}
    \text{HiResCAM}_c(x) =  \text{ReLU}\left(\sum_{i=1}^{N_l} W^i \odot A^i\right).
\end{equation}
GradCAM Elementwise \cite{jacobgilpytorchcam} modifies HiResCAM by adding a ReLU operation to the front of the summation:
\begin{equation}\label{eq:gradcam-e}
    \text{GradCAM-E}_c(x) =  \text{ReLU}\left(\sum_{i=1}^{N_l} \text{ReLU}(W^i \odot  A^i)\right).
\end{equation}
LayerCAM~\cite{jiang2021layercam} further refines HiResCAM by adding a ReLU to the gradients before the Hadamard product:
\begin{equation}\label{eq:layercam}
    \text{LayerCAM}_c(x) =  \text{ReLU}\left(\sum_{i=1}^{N_l} \text{ReLU}(W^i) \odot A^i\right).
\end{equation}
XGradCAM \cite{fu2020axiom} was introduced to enhance the sensitivity and consistency properties of GradCAM. It generates the heatmap as follows, where \(\overline{X}\) denotes the mean of \(X\):
\begin{equation}
    \text{XGradCAM}_c(x) =  \text{ReLU}\left(\sum_{i=1}^{N_l} \frac{\overline{W^i \odot  A^i}}{\overline{A^i}} A^i\right).
\end{equation}



GradCAM++~\cite{chattopadhay2018grad} generates heatmaps by focusing on positive gradients, like LayerCAM, while incorporating higher-order derivatives. For CNNs with a GAP layer, it provides a closed-form solution for the weights. To avoid high-order derivative computation, it approximates second and third-order derivatives using squared and cubed gradients, which holds only when logits are passed through an exponential function.
LIFTCAM \cite{liftcam2021jung} estimates weights by evaluating the contribution of each activation map using DeepLIFT through a single backward pass. The accuracy of the weights is limited by the precision of the explanations provided by DeepLIFT.

Thus far, all introduced CAM methods are gradient-based (CAM can also be considered gradient-based \cite{selvaraju2017grad}). However, noisy or vanishing gradients in deep networks can undermine meaningful explanations \cite{wang2020score}, as also noted in Section \ref{subsec:softmax}. To address this, gradient-free methods have been proposed. 
ScoreCAM \cite{wang2020score} generates heatmaps by overlapping normalized activation maps with the input image and determines weights by applying softmax to the \(N_l\) output logits corresponding to the target class. AblationCAM \cite{abalation2020desai} assigns weights based on the decrease in the target output when each associated activation map is set to zero. Although gradient-free CAMs sometimes provide a more accurate explanation, the requirement of \(N_l\) forward propagations, typically hundreds of times \cite{torcham2020} more time-consuming and resource-intensive than gradient-based CAMs, hinders their use with large datasets.

In summary, CAM methods are known for their efficiency and are widely used in explainability. However, they share a common limitation: reliance on heuristics and a lack of a solid theoretical framework. This paper aims to address this issue.

\subsection{Shapley Value-Based Methods}
The Shapley value \cite{shapley1953value}, a fundamental concept from cooperative game theory \cite{algaba2019handbook}, has become widely utilized in machine learning for attributing contributions and ensuring fairness \cite{roz2022shapleysurvey}. It is particularly valued for its ability to fairly allocate the total utility $U(D)$ (such as revenue, cost, or even the output probability of a model) among all players in $D$ by assessing each player’s marginal contribution across all possible player subsets. Specifically, the Shapley value is the unique solution concept that satisfies the following four axioms:

\begin{itemize}
    \item \textbf{Dummy player:} If \( U(S \cup \{i\}) = U(S)\) for all \( S \subseteq D \setminus \{i\} \), then \( \phi(i; U) = 0 \).
    \item \textbf{Symmetry:} If \( U(S \cup \{i\}) = U(S \cup \{j\}) \) for all \( S \subseteq D \setminus \{i, j\} \), then \( \phi(i; U) = \phi(j; U) \).
    \item \textbf{Efficiency:} \( \sum_{i \in D} \phi(i; U) = U(D) -U(\emptyset) \).
    \item \textbf{Linearity:} For utility functions \( U_1, U_2 \) and any \( \alpha_1, \alpha_2 \in \mathbb{R} \), \( \phi(i; \alpha_1 U_1 + \alpha_2 U_2) = \alpha_1 \phi(i; U_1) + \alpha_2 \phi(i; U_2) \).
\end{itemize}

\begin{definition}[\textbf{Shapley, 1953}\cite{shapley1953value}]
\label{def:shapley}
\textit{Given a player set $D$ with $n = |D|$ and a utility function $U$, the Shapley value for each player $i \in D$ is defined as:} 
\begin{equation}\label{eq:shapley}
    \phi\left(i; U \right) = \frac{1}{n} \sum_{k=1}^{n} \binom{n-1}{k-1}^{-1} 
    \sum_{\substack{S \subseteq D\setminus \{i\} \\ |S|=k-1}} \left[ U(S \cup \{i\}) - U(S) \right].
\end{equation}
A single-valued solution concept satisfies the axioms of dummy player, symmetry, efficiency, and linearity if and only if it is the Shapley value.
\end{definition}

A more intuitive form for the Shapley value is:
\begin{equation}
\phi(i; U) = \mathbb{E}_{\pi \sim \Pi} \left[ U\left( S_{\pi}^i \cup \{i\} \right) - U\left( S_{\pi}^i \right) \right].
\end{equation}
where \( \pi \sim \Pi \) represents a uniformly random permutation of the set of players \(D\), and \( S_{\pi}^i \) denotes the set of players that precede player \( i \) in permutation \( \pi \). This form reveals that the Shapley value of player \( i \) measures the expected marginal contribution across all possible subsets, as shown in Fig. \ref{fig:shapley_value}.
In the context of feature attribution, the Shapley value ensures a fair distribution of utility among features. The axiom of \emph{Dummy Player} assigns zero contribution to features that contribute nothing in any subset. \emph{Symmetry} guarantees that features with equal marginal utility across all subsets receive equal contribution. \emph{Efficiency} ensures that the total utility is fully distributed among all features. \emph{Linearity} implies that if the utility function is a linear combination of two functions, the contribution assigned to a feature is also the linear combination of its contributions with respect to those functions.

Next, we introduce two prominent applications of the Shapley value in explainable AI: SHAP \cite{LundbergL17SHAP} and Data Shapley \cite{zou2019datashapley}. These methods study feature attribution in model inference and data attribution in model training, respectively.

SHAP (SHapley Additive exPlanations) \cite{LundbergL17SHAP} treats the model's output as a utility function and the \( n \) input features as players, approximating features' Shapley value to explain model predictions. SHAP offers a qualitative understanding of feature contributions and is widely adopted by data analysis and healthcare \cite{roz2022shapleysurvey}. However, computing the exact Shapley values requires \(\mathcal{O}(2^n)\) model forward passes. Approximation techniques like Kernel SHAP \cite{LundbergL17SHAP} address this by solving a weighted least squares problem, but the cost remains significant, scaling as at least \(\mathcal{O}(c n)\), where \( c \) is non-negligible. In contrast, CAM methods require only a single forward and backward pass, making them more practical for high-dimensional data like images. Nevertheless, SHAP remains a cornerstone in explainable AI due to the Shapley value, which is widely regarded as the fairest method for utility allocation \cite{algaba2019handbook, roz2022shapleysurvey}, whereas CAMs lack a similarly theoretical basis.

\vspace{-0.20cm}
\begin{figure}[htbp]
    \centering
    \includegraphics[width=0.35\textwidth]{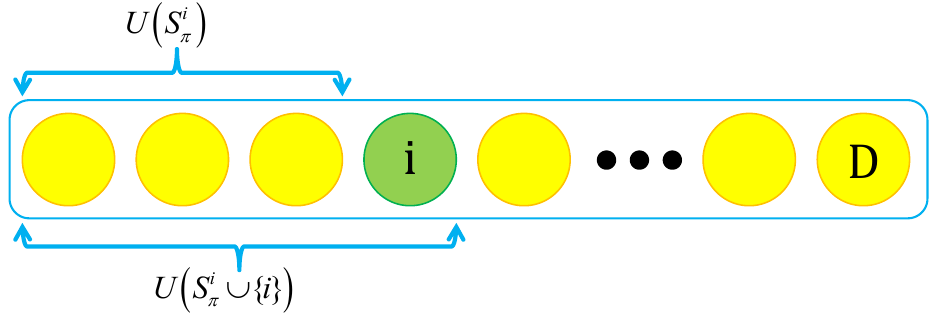} 
    \caption{Illustration of the calculation process of Shapley value}
    \label{fig:shapley_value}
\end{figure}

Data Shapley \cite{zou2019datashapley} focuses on data attribution during model training by quantifying the contribution of each data point to a model’s performance. This approach aids in identifying both valuable and noisy data points and supports downstream tasks such as data selection, acquisition, and cleaning. It also supports the development of data markets \cite{Mazumder2023dataperf}. Data Shapley employs Monte Carlo methods \cite{zou2019datashapley} to estimate the value of each data point, with $\mathcal{O}(n^2 \log n)$ time of model retraining, where $n$ is the number of data points. This complexity limits its scalability to larger datasets such as CIFAR-10. 
A recent work, In-Run Data Shapley \cite{wang2024data}, explores the use of second-order Taylor expansions to approximate the utility function of model performance, enabling a closed-form solution for estimating the Shapley value of each training data point.

In summary, Shapley value-based methods are recognized for their solid theoretical foundations, but computational challenges limit their widespread use. In this work, we extend the idea of directly estimating the utility function from In-Run Data Shapley \cite{wang2024data} to feature attribution. The key insight is that if the inference result for a subset of features can be expressed analytically, a closed-form solution for each feature's Shapley value may be derived, enabling a more computationally efficient approach.

\section{Content Reserved Game-theoretic Explainer}
\label{sec:crg}
In this section, we first formally define the CRG Explainer. We then use this theoretical framework to clarify the theoretical foundations of GradCAM and HiResCAM, support the development of new CAM methods like ShapleyCAM and its variants, and answer the question of ReLU placement. Additionally, as the choice of the utility function is central to cooperating game theory, we analyze the advantages and limitations of pre and post-softmax scores on the explanations, reveal their relationship, and propose the ReST utility function to overcome these limitations.

\subsection{Definition of Content Reserved Game-Theoretic Explainer}
The key difference between SHAP and CAMs lies in their inputs. SHAP operates on the raw image, computing Shapley values for individual pixels, while CAMs use \(N\) \(d\)-dimensional activation maps \(\{A^i\}_{i=1}^N\) extracted from a target layer during inference to generate a \(d\)-dimensional explanation. The motivation behind the CRG Explainer is to reinterpret CAMs in a manner similar to SHAP.

Since each activation map can be viewed as a transformation of the original image that preserves its spatial structure \cite{zhou2016learning}, we treat \( \{A^i\}_{i=1}^N \) as \(N\) variants of a downsampled \(d\)-dimensional raw image. The cooperative game of model prediction then occurs at each pixel of this downsampled image, or equivalently, at each group of the same position in \( \{A^i\}_{i=1}^N \). Next, we present our theoretical framework for CAM methods:

\begin{definition}
\label{def:CRG}
Given \(N\) \(d\)-dimensional vectors \(\{A^i\}_{i=1}^N\) that cooperate to achieve the scalar utility \( U(D)\), the group of \(j\)-th elements in each \(A^i\), i.e., \(\{A^i_j\}_{i=1}^N\), is treated as the \(j\)-th player, with the player set $D := \{ j \}_{j=1}^d$. If the Shapley value of the \(j\)-th player is \(\phi(j;U) = \sum_{i=1}^N W^i_j A^i_j\), a \textbf{Content Reserved Game-theoretic (CRG) Explainer} generates a \(d\)-dimensional explanation vector \( E = \sum_{i=1}^N g(W^i) \odot A^i \), where \(W^i\) is a \(d\)-dimensional vector, \(g\) is a mapping function, and \(\odot\) denotes the Hadamard product. The explainer is classified as follows:
\begin{itemize}
    \item \textbf{Type-I CRG Explainer} satisfies the \textbf{game-theoretic} property: \(g(W^i) = W^i\). 
    \item \textbf{Type-II CRG Explainer} satisfies the \textbf{content reserved} property: \(g(W^i) = \overline{W^i} \mathbf{1}_d\), where \(\overline{W^i}\) is the mean of \(W^i\).
    \item \textbf{Optimal CRG Explainer} satisfies both \textbf{content reserved} and \textbf{game-theoretic} properties: \(\forall i, W^i = \overline{W^i} \mathbf{1}_d\).
\end{itemize}
\end{definition}

The Type-I CRG Explainer produces a faithful explanation of the cooperative game among the \(d\) players, as the value of the \(j\)-th pixel in \(E\) is also the Shapley value of the \(j\)-th player: \(E_j =  [\sum_{i=1}^N W^i \odot A^i]_j  = \phi(j;U)\). Additionally, since \( \langle W^i \odot A^i, \mathbf{1}_d \rangle = \langle W^i, A^i \rangle \), we have \( \langle E, \mathbf{1}_d \rangle = \langle \sum_{i=1}^N W^i \odot A^i, \mathbf{1}_d \rangle = \sum_{i=1}^N \langle W^i, A^i \rangle =\sum_{j=1}^d \phi(j;U) =  U(\{A^i\}_{i=1}^N) - U(\emptyset)\), where the last equality follows from the Efficiency axiom. Hence, the explanation produced is a rearrangement of the total utility for each pixel.
However, \(W^i\) may be inaccurate due to noisy or vanishing gradients \cite{wang2020score}, while \(A^i\) is more dependable as it maintains the spatial structure of the raw image \cite{zhou2016learning}. The Type-II CRG Explainer seeks to reserve more content from \(\{A^i\}_{i=1}^N\) by weighting them with scalars \( \{\overline{W^i}\}_{i=1}^N\), albeit with a potential sacrifice of the game-theoretic property. 
Lastly, the Optimal CRG Explainer occurs when the layer before the GAP layer is chosen as the target layer, as all gradients on the same activation map are identical and equal to \(1/d\) of the gradient of the pooled score in the GAP layer. Detailed proof is available in \cite{selvaraju2017grad}. Here, ``optimal" indicates that this CRG Explainer simultaneously satisfies both the content reserved and game-theoretic properties, achieving the best of both worlds.

\subsection{Theoretical Foundation of GradCAM and HiResCAM}

Within the CRG Explainer, we establish the theoretical foundation for GradCAM and HiResCAM by approximating the utility function with a first-order Taylor expansion.

\begin{theorem}\label{thm:hirescam}
When using first-order Taylor expansion to approximate the utility function, HiResCAM is a Type-I CRG Explainer, and GradCAM is a Type-II CRG Explainer. Both GradCAM and HiResCAM are Optimal CRG Explainers if the target layer is the layer preceding the GAP layer.
\end{theorem}

\begin{proof}
\noindent \textbf{Notation:} Define \( X_D \in \mathbb{R}^{1 \times Nd} \) as \( X_D := \left[ A^1 \mid A^2 \mid \cdots \mid A^N \right] \), where \( A^i \in \mathbb{R}^{1 \times d} \). Let \( X^i \in \mathbb{R}^{1 \times Nd} \) be the vector with \( A^i \) at its original location and zeros elsewhere: \( X^i := \left[ \vec{0} \mid \cdots \mid A^i \mid \cdots \mid \vec{0} \right] \), where \( \vec{0} \in \mathbb{R}^{1 \times d} \). Define \( X_j \in \mathbb{R}^{1 \times Nd} \) as \( X_j := \left[ 0, A^1_j, 0 \mid \cdots \mid 0, A^i_j, 0 \mid \cdots \mid 0, A^N_j, 0 \right] \) retaining only the \( j \)-th element of each \( A^i \), and let \( X_S := \sum_{j=1,j\in S}^d X_j \) for a subset \( S \subseteq D = \{ j \}_{j=1}^d \). The utility  function values are represented by \( U(X_S) \) and \( U(X_D) \), with \( U'(X_D) \) as the gradient and \( H_D \) as the Hessian matrix at \( X_D \). 
The Taylor expansion of \( U(X_S) \) includes \( U_1(X_S) \) for the first-order term and \( U_2(X_S) \) for the second-order term. 

\noindent \textbf{Derivations:} We begin by applying the first-order Taylor expansion of the functions \( U(X_S) \) at \( X_D \) for all \(S\) as in \cite{wang2024data}:
\begin{align*}
    U(X_S) &\approx U(X_D) + \underbrace{U'(X_D) (X_S - X_D)}_{U_1(X_S)}, \\
    U(X_{S \cup j}) &
    \approx U(X_D) + \underbrace{U'(X_D)(X_S - X_D + X_j)}_{U_1(X_{S \cup j})}.
\end{align*}

Then, consider the difference:
\begin{align*}
    &U(X_{S \cup j}) - U(X_S) = \underbrace{U'(X_{D}) X_j}_{U_1(X_{S \cup j}) - U_1(X_S) }.
\end{align*}

To compute \( \phi(j; U) \), we use the Equation \eqref{eq:shapley}, and substitute \( U(X_{S \cup j}) - U(X_S) \) with \(U'(X_D) X_j \):
\begin{align}\label{eq:U1}
    &\phi(j; U) \nonumber\\&= \frac{1}{d} \sum_{k=1}^{d} \binom{d-1}{k-1}^{-1} 
    \sum_{\substack{S \subseteq D\setminus \{j\} \\ |S|=k-1}} \left( U(X_{S \cup j}) - U(X_S)\right) \nonumber \\
    &= \frac{1}{d} \sum_{k=1}^{d} \binom{d-1}{k-1}^{-1} 
    \sum_{\substack{S \subseteq D\setminus \{j\} \\ |S|=k-1}} U'(X_D) X_j \nonumber \\
    &= \frac{1}{d} U'(X_D) X_j \sum_{k=1}^{d} \binom{d-1}{k-1}^{-1} 
    \sum_{\substack{S \subseteq D\setminus \{j\} \\ |S|=k-1}} 1 \nonumber \\
    &= \frac{1}{d} U'(X_D) X_j \sum_{k=1}^{d} \binom{d-1}{k-1}^{-1} \binom{d-1}{k-1} \nonumber \\
    &= \frac{1}{d} U'(X_D) X_j \sum_{k=1}^{d} 1 = U'(X_D) X_j.
\end{align}

Therefore, Shapley values under the first-order approximation of \( U \) are given by:
\begin{align}\label{eq:1st_shapley}
    \phi(j; U) = U'(X_D) X_j = \sum_{i=1}^N [U'(X_D)]^i_j A^i_j.
\end{align}

By integrating Equation \eqref{eq:1st_shapley} into the definition of CRG Explainer, we deduce that the coefficient \(W_j^i\) associated with \(A_j^i\) is expressed as \([U'(X_D)]^i_j\), representing the gradients of the utility function concerning \(A_j^i\). Thus, according to the definition of the CRG Explainer, the theorem is substantiated.
\end{proof}

It has been proven in \cite{selvaraju2017grad} that using the pre-softmax score as the utility function results in GradCAM and CAM heatmaps being identical, differing only by a constant factor eliminated during normalization. Thus, we have the following corollary:

\begin{corollary}\label{thm:cam}
CAM \cite{zhou2016learning} is the Optimal CRG Explainer when using the layer preceding the GAP as the target layer and the pre-softmax score as the utility function.
\end{corollary}

\begin{algorithm}[t]
\renewcommand{\algorithmicrequire}{\textbf{Input:}}
\renewcommand{\algorithmicensure}{\textbf{Output:}}
\begin{algorithmic}[1]
\Require{Input image \(x\), neural network \(f(\cdot)\), target class \(c\), target layer \(l\).}
\Ensure{Heatmap.}
\State Forward pass: Compute logits \(y\) for the input \(x\) using \(f(\cdot)\), and save the activation maps \( \{A^i\}_{i=1}^{N_l} \) from the \(l\) target layer. Concatenate these maps to obtain \( X_D := \left[ A^1 \mid A^2 \mid \cdots \mid A^{N_l} \right] \).
\State Compute the ReST: \( U_{\text{ReST}} = y^c + \ln(\text{softmax}(y)^c) \).
\State Backward pass: Compute the gradient of \(U_{\text{ReST}}\) with respect to the activation maps: \( U'(X_D) = \partial U_{\text{ReST}}/\partial X_D \).
\State Backward pass: Compute the Hessian matrix of \(U_{\text{ReST}}\) with respect to the activation maps: \( H_D = \partial^2 U_{\text{ReST}}/\partial X_D^2 \).
\State Compute the weights: \(W^i = \left[ U'(X_D) - \frac{1}{2} X_D^\top H_D \right]^i\).
\If{use ShapleyCAM}
    \State \textbf{return} \(\text{ReLU}\left( \sum_{i=1}^{N_l} \overline{W^i} A^i \right)\).
\ElsIf{use ShapleyCAM-H}
    \State \textbf{return} \( \text{ReLU} \left( \sum_{i=1}^{N_l} (W^i \odot A^i) \right)\).
\ElsIf{use ShapleyCAM-E}
    \State \textbf{return} \(\text{ReLU}\left( \sum_{i=1}^{N_l} \text{ReLU}(W^i \odot A^i) \right)\).
\EndIf
\end{algorithmic}
\caption{ShapleyCAM and its variants}
\label{alg:shapleycam}
\end{algorithm}

\subsection{ShapleyCAM Algorithms}

Within the CRG Explainer, we propose ShapleyCAM and its variants by approximating the utility function via a second-order Taylor expansion. Applying \( W^i_j = [U'(X_D) - \frac{1}{2} X_D^\top H_D]^i_j \) from Equation \eqref{eq:finalshapley} to Equations \eqref{eq:gradcam}, \eqref{eq:hirescam}, and \eqref{eq:gradcam-e}, which correspond to GradCAM, HiResCAM, and GradCAM-E, respectively, we derive ShapleyCAM, ShapleyCAM-H, and ShapleyCAM-E. Algorithm \ref{alg:shapleycam} presents the detailed algorithms.


\begin{theorem}\label{thm:2ndcrg}
When using second-order Taylor expansion to approximate the utility function, ShapleyCAM-H is a Type-I CRG Explainer, and ShapleyCAM is a Type-II CRG Explainer. Both ShapleyCAM and ShapleyCAM-H are Optimal CRG Explainers if the target layer is the layer preceding GAP.
\end{theorem}
\begin{proof}
Here, we apply the second-order Taylor expansion of the functions \( U(X_S) \) at \( X_D \) for all \(S\) as in \cite{wang2024data}:
\begin{align*}
    U(X_S) &\approx U(X_D) + \underbrace{U'(X_D) (X_S - X_D)}_{U_1(X_S)} \\&\quad+ \underbrace{\frac{1}{2} (X_S - X_D)^\top H_D (X_S - X_D)}_{U_2(X_S)}, \\
    U(X_{S \cup j}) &
    \approx U(X_D) + \underbrace{U'(X_D)(X_S - X_D + X_j)}_{U_1(X_{S \cup j})} \\ \quad+&\underbrace{ \frac{1}{2} (X_S - X_D + X_j)^\top H_D (X_S - X_D + X_j)}_{U_2(X_{S \cup j})}.
\end{align*}

Then, consider the difference:
\begin{align*}
    &U(X_{S \cup j}) - U(X_S) 
    \\&= \underbrace{U'(X_D) X_j}_{U_1(X_{S \cup j}) - U_1(X_S) } \\&+\underbrace{ \frac{1}{2}X_j^\top H_D X_j - X_D^\top H_D X_j + X_S^\top H_D X_j}_{U_2(X_{S \cup j}) - U_2(X_S) }.
\end{align*}

Using the Linearity property of the Shapley value, we can decompose the Shapley value of \(j\) under \(U\) into the sum of its Shapley values under \(U_1\) and \(U_2\):
\begin{align}\label{eq:linear}
     \phi(j;U_1 + U_2) = \phi(j;U_1) + \phi(j;U_2).
\end{align}


Similar to Equation \eqref{eq:U1}, we substitute \( U_2(X_{S \cup j}) - U_2(X_S) \) with \(\frac{1}{2}X_j^\top H_D X_j - X_D^\top H_D X_j + X_S^\top H_D X_j \):

\begin{align*}
     &\phi(j; U_2) = \frac{1}{d} \sum_{k=1}^{d} \binom{d-1}{k-1}^{-1} 
    \sum_{\substack{S \subseteq D\setminus \{j\} \\ |S|=k-1}}  X_S^\top H_D X_j \\
     + & \frac{1}{d} \sum_{k=1}^{d} \binom{d-1}{k-1}^{-1} 
    \sum_{\substack{S \subseteq D\setminus \{j\} \\ |S|=k-1}} (  \frac{1}{2}X_j^\top H_D X_j-X_D^\top H_D X_j)\\
    &= \frac{1}{d}  \sum_{k=2}^{d} \binom{d-1}{k-1}^{-1} \sum_{i\in D\setminus{j}}\sum_{\substack{S \subseteq D\setminus \{i,j\} \\ |S|=k-2}} X_i^\top H_D X_j\\
    + &\frac{1}{d} \sum_{k=1}^{d} \binom{d-1}{k-1}^{-1} 
    \sum_{\substack{S \subseteq D\setminus \{j\} \\ |S|=k-1}} (  \frac{1}{2}X_j^\top H_D X_j-X_D^\top H_D X_j),
\end{align*}

The last step relies on \( X_S = \sum_{i \in S} X_i \), which allows us to express the sum over all possible subsets \( S \) in terms of the individual elements \( X_i \) within those subsets. Then:
\begin{align}\label{eq:U2}
    &\phi(j; U_2) = \frac{1}{d}  \sum_{k=2}^{d} \binom{d-1}{k-1}^{-1} \sum_{i\in D\setminus{j}}\binom{d-2}{k-2} X_i^\top H_D X_j \nonumber\\
    &+ \frac{1}{d} ( \frac{1}{2} X_j^\top H_D X_j -X_D^\top H_D X_j )\sum_{k=1}^{d} \binom{d-1}{k-1}^{-1} \binom{d-1}{k-1} \nonumber \\
    &= \frac{1}{d}   \sum_{k=2}^{d} \binom{d-1}{k-1}^{-1} \binom{d-2}{k-2}\left(\sum_{i\in D\setminus{j}}  X_i^\top H_D X_j\right) \nonumber\\
    &\quad + \frac{1}{d} ( \frac{1}{2} X_j^\top H_D X_j -X_D^\top H_D X_j)\sum_{k=1}^{d} 1 \nonumber\\
    &= \frac{1}{d} \sum_{k=2}^{d} \frac{k-1}{d-1}(X_{D\setminus{j}}^\top H_D X_j) + (\frac{1}{2} X_j^\top H_D X_j -X_D^\top H_D X_j )\nonumber\\
    &= \frac{\sum_{k=2}^{d} (k-1)}{d(d-1)}X_{D\setminus{j}}^\top H_D X_j + \frac{1}{2} X_j^\top H_D X_j -X_D^\top H_D X_j \nonumber\\
    &= \frac{1}{2}X_{D\setminus{j}}^\top H_D X_j + \frac{1}{2} X_j^\top H_D X_j -X_D^\top H_D X_j \nonumber\\
    &= \frac{1}{2} X_D^\top H_D X_j -X_D^\top H_D X_j \nonumber\\
    &= -\frac{1}{2} X_D^\top H_D X_j.
\end{align}
Therefore, Shapley values under the utility function \( U_2 \) are:
\begin{align}
    \phi(j; U_2) =  -\frac{1}{2} X_D^\top H_D X_j = \sum_{i=1}^N [-\frac{1}{2}X_D^\top H_D]^i_j A^i_j.
\end{align}

Equation \eqref{eq:U1} also represents the computation of \(\phi(j; U_1)\). By substituting Equations \eqref{eq:U1} and \eqref{eq:U2} into Equation \eqref{eq:linear}, we derive the final expression for \(\phi(j;U)\) under the second-order approximation of \(U\):
\begin{align}\label{eq:finalshapley}
    \phi(j;U) &=  U'(X_D) X_j - \frac{1}{2} X_D^\top H_D X_j\notag\\
    &=\sum_{i=1}^N [U'(X_D)-\frac{1}{2}X_D^\top H_D]^i_j A^i_j.
\end{align}

By integrating Equation \eqref{eq:finalshapley} into the CRG Explainer, we deduce that the coefficient \(W_j^i\) associated with \(A_j^i\) is expressed as \( W^i_j = [U'(X_D) - \frac{1}{2} X_D^\top H_D]^i_j \). Thus, according to the definition of the CRG Explainer, the theorem is substantiated.
\end{proof}

Previous work like LIFTCAM \cite{liftcam2021jung} has aimed to combine the Shapley value with CAM, using Shapley values from DeepLIFT to weight activation maps. In contrast, ShapleyCAM leverages the Shapley value to reframe the CAMs, treating the pixels of the downsampled \(d\)-dimensional raw image as players, aligning more closely with SHAP. However, SHAP aims to explain the entire highly complex model, while ShapleyCAM narrows its focus to the layers between the target layer and the output's utility function, making the explanation process more manageable. Additionally, ShapleyCAM employs a derived closed-form Shapley value, thus circumventing the repeated inference that SHAP necessitates.

For the complexity of the ShapleyCAM algorithm, although Hessian matrix computation is often considered computationally intensive, ShapleyCAM only requires the Hessian-vector product of \( H_D \) and \( X_D \), which is efficiently supported in modern deep learning frameworks like PyTorch and JAX \cite{da2024howtohvp}. This allows ShapleyCAM to run with just one extra backward pass  \cite{da2024howtohvp} compared to GradCAM, making it scalable for large datasets as well.

\subsection{Utility Function: Pre or Post-Softmax Scores? Both!}
\label{subsec:softmax}

The choice between pre-softmax and post-softmax scores for generating explanations remains debated \cite{lerma2023pre}. While most methods \cite{zhou2016learning,selvaraju2017grad,chattopadhay2018grad} use pre-softmax scores, this has not been fully analyzed from a decision-making perspective. Intuitively, combining the input with the CAM heatmap should increase the confidence or probability for the target class. However, using pre-softmax scores can lead to a situation where the region of non-target classes is also highlighted, which, paradoxically, might reduce the new probability of the target class. For example, as shown in the top-left in Fig. \ref{fig:softmax}, when the \textit{tiger cat} is the target class, GradCAM with pre-softmax incorrectly highlights the \textit{boxer}. 

Using post-softmax scores can avoid the phenomenon. As shown in the top-middle of Fig. \ref{fig:softmax}, GradCAM with post-softmax does not highlight the \textit{boxer}. Besides, when a  \(\ln\)  function is applied after softmax \cite{lerma2023pre}, the resulting utility function matches the cross-entropy loss used during model training.

Next, we reveal the relationship between the generated heatmaps when utilizing the pre and post-softmax score as utility functions, respectively, from a theoretical perspective. Assume GradCAM uses the activation maps \( \{A^i\}_{i=1}^{N_l}\) from target layer \(l\) (regardless of which layer is chosen) and gradients \( \frac{\partial y^c}{\partial A^i} \) or \( \frac{\partial p^c}{\partial A^i} \) to generate the heatmap \(E_c^{\text{pre}} \) or \(E_c^{\text{post}} \) before the ReLU operation. Here, \( y^c \) represents the logit for the target class  \(c\in [1, \dots, C]\), and \( p^c =\text{softmax}(y)^c \) denotes the corresponding probability, then the following theorem holds:

\begin{theorem}\label{thm:softmax_gradcam_equiv}
The heatmap generated by GradCAM using the post-softmax score is equivalent to the ensemble of  \(C\) heatmaps generated by GradCAM using the pre-softmax score.
\end{theorem}

\begin{proof}

Suppose that GradCAM produces the heatmap before the ReLU operation as \( E_c^{\text{post}} = \frac{1}{N_l} \sum_{i=1}^{N_l} g\left(\frac{\partial p^c}{\partial A^i}\right) \odot A^i \) when using the post-softmax score, and as \( E_c^{\text{pre}} = \frac{1}{N_l} \sum_{i=1}^{N_l} g\left(\frac{\partial y^c}{\partial A^i}\right) \odot A^i \) when using the pre-softmax score. Here, \( g(A) = \frac{\langle A, \mathbf{1}_d \rangle}{d} \mathbf{1}_d \) replaces each element of the \( d \)-dimensional vector \( A \) with its mean.

To compute the gradient \( \frac{\partial p^c}{\partial A^i} \), we use the chain rule: \(\frac{\partial p^c}{\partial A^i} = \sum_{k=1}^{C} \frac{\partial p^c}{\partial y^k} \frac{\partial y^k}{\partial A^i}\).
With the facts that \( \frac{\partial p^c}{\partial y^c} = (1- p^c)p^c \) and \( \frac{\partial p^c}{\partial y^k} = - p^k p^c \) when \(k\ne c\), we have
\(\frac{\partial p^c}{\partial A^i} = (1 - p^c)p^c \frac{\partial y^c}{\partial A^i} - \sum_{k=1, k \neq c}^{C} p^c p^k \frac{\partial y^k}{\partial A^i}\),
or equivalently, \(\frac{\partial p^c}{\partial A^i} = p^c \sum_{k=1, k \neq c}^{C} p^k \left( \frac{\partial y^k}{\partial A^i} - \frac{\partial y^c}{\partial A^i} \right)\) because \(\sum_{k=1}^C p^k = 1\).

It is easy to verify that \(g(A)\) is a linear transformation of \( A \), because  \(g(A)\) enjoys the properties \( g(A_1 + A_2) = \frac{\langle A_1+A_2, \mathbf{1}_d\rangle}{d} \mathbf{1}_d  =  g(A_1) + g(A_2) \) and \( g(\lambda A) = \frac{\langle\lambda A , \mathbf{1}_d\rangle}{d} \mathbf{1}_d  = \lambda g(A) \). Thus, \(
g( \frac{\partial p^c}{\partial A^i} ) = p^c \sum_{k=1, k \neq c}^{C} p^k ( g ( \frac{\partial y^k}{\partial A^i} ) - g ( \frac{\partial y^c}{\partial A^i} ))\). 

Then we have \(
E_c^{\text{post}} = \frac{1}{N_l} \sum_{i=1}^{N_l} g( \frac{\partial p^c}{\partial A^i}) \odot A^i = \frac{1}{N_l} \sum_{i=1}^{N_l} (p^c \sum_{k=1, k \neq c}^{C} p^k ( g ( \frac{\partial y^k}{\partial A^i} ) - g ( \frac{\partial y^c}{\partial A^i} ))) \odot A^i = p^c \sum_{k=1, k\ne c}^{C} p^k ( \frac{1}{N_l} \sum_{i=1}^{N_l} ( g ( \frac{\partial y^k}{\partial A^i}) - g ( \frac{\partial y^c}{\partial A^i} ) ) \odot A^i )= p^c \sum_{k=1, k\ne c}^{C} p^k ( \frac{1}{N_l} \sum_{i=1}^{N_l} g ( \frac{\partial y^k}{\partial A^i}) \odot A^i - \frac{1}{N_l} \sum_{i=1}^{N_l}g ( \frac{\partial y^c}{\partial A^i} ) \odot A^i )
\)
and eventually reduces to:
\begin{equation}\label{eq:softmax_gradcam_equiv}
    E_c^{\text{post}} = p^c \sum_{k=1, k\ne c}^{C} p^k (E_c^{\text{pre}} - E_k^{\text{pre}}).
\end{equation}
\end{proof}

In the proof of Theorem \ref{thm:softmax_gradcam_equiv}, we use the heatmap before applying the ReLU operation as the explanation. ReLU emphasizes positive regions \cite{jacobgilpytorchcam}, but some areas may actually contribute negatively to the prediction, suggesting that the original explanation should be the heatmap before ReLU.
Actually, our proof can be generalized to other CAM methods if \( g(X)\) is a linear transformation. For instance, Theorem \ref{thm:softmax_gradcam_equiv} also holds for HiResCAM where \( g(X) = X \).

Not only does Theorem \ref{thm:softmax_gradcam_equiv} establish a connection between heatmaps generated by GradCAM using pre-softmax score and post-softmax score, but it also provides insight into the effectiveness of the post-softmax method. Specifically, it shows that the post-softmax method uses the difference between the target class heatmap and other class heatmaps, generated by GradCAM with pre-softmax scores, to produce the final explanation: Although \(E_c^{\text{pre}}\) may highlight regions belonging to another class \(b\), the subtraction of \(E_b^{\text{pre}}\) removes their influence.

Furthermore, Theorem \ref{thm:softmax_gradcam_equiv} also reveals a limitation of using post-softmax scores: When a neural network is extremely confident in its prediction (i.e., \(p^c\) is really close to 1), the gradients of the probabilities for other classes become very small (i.e., \(p^k\) is also really close to \(0\) for \(k\ne c\)), causing the generated heatmap \( E_c^{\text{post}}\) to diminish towards zero, as shown in Equation \eqref{eq:softmax_gradcam_equiv}. This phenomenon is commonly referred to as gradient vanishing \cite{lerma2023pre}. For instance, the bottom row of Fig. \ref{fig:softmax} shows an image where ResNet-18 predicts \textit{axolotl} with a probability greater than $1 - 10^{-15}$. In this case, post-softmax suffers from gradient vanishing and fails to highlight the \textit{axolotl}, while pre-softmax does not have this issue and correctly highlights the \textit{axolotl}.

To address this issue, inspired by the idea of residual \cite{he2016deep}, we propose Residual Softmax Target-Class (ReST) utility function using both pre and post-softmax scores:
\begin{equation}
    U_{\text{ReST}} = y^c + \ln(\text{softmax}(y)^c).
\end{equation}
Then the heatmap produced by GradCAM with ReST is:
\begin{equation}\label{eq:ReST}
    E_c^{\text{ReST}} = E_c^{\text{pre}} + \sum_{k=1, k \neq c}^{C} p^k \left(E_c^{\text{pre}} - E_k^{\text{pre}}\right).
\end{equation}
ReST mitigates the negative impact of gradient vanishing by incorporating an extra \(E_c^{\text{pre}}\) term, while also maintaining a stable focus on the target class similar to post-softmax. As shown in the right column of Fig. \ref{fig:softmax}, GradCAM with ReST or post-softmax consistently focuses on the \textit{tiger cat}, while pre-softmax incorrectly highlights the \textit{boxer}. Additionally, GradCAM with ReST or pre-softmax can still correctly focus on the \textit{axolotl} when post-softmax fails due to gradient vanishing. Therefore, ReST can be viewed as a method that combines the advantages of pre-softmax and post-softmax scores while avoiding their respective drawbacks. In ShapleyCAM and its variants, we adopt the ReST utility function. It is used in all experiments, except in the ablation study evaluating ReST.


\subsection{Shapley Value and ReLU}

ReLU was introduced by GradCAM as a heuristic operation to eliminate negative regions in explanations and has been adopted by subsequent works. Methods like LayerCAM, GradCAM++, and GradCAM-E also apply ReLU to gradients or other components, but these methods may produce problematic explanations, as shown in Fig. \ref{fig:main}.

Within the CRG Explainer, the question of ReLU placement can be answered from a Shapley value perspective. Each pixel in the heatmap represents its Shapley value, with Shapley values greater than 0 indicating a positive contribution \cite{jia2023banzhaf} to the model's output utility and thus should be highlighted. Therefore, the most logical placement for ReLU is outside the summation, consistent with GradCAM and HiResCAM.

\section{Experiments}

This section evaluates various CAM methods across twelve distinct neural network architectures, utilizing two types of target layers and six metrics to assess the quality of the explanations. All experiments were performed on a server equipped with an Intel(R) Xeon(R) Gold 6326 CPU @ 2.90GHz and an NVIDIA A40 GPU. And we employ ReST as the utility function in all experiments, except for the ablation study of ReST.
\subsection{Setup} 
\subsubsection{Dataset}
Unlike previous studies that evaluated CAMs on randomly selected images, our experiments were conducted on the full ImageNet validation set (ILSVRC2012)~\cite{russakovsky2015imagenet}, which consists of 50,000 images spanning 1,000 distinct object categories. Each image was resized and center-cropped to 224 × 224 pixels, and subsequently normalized using the mean and standard deviation computed from the ImageNet training set.

\subsubsection{Networks}
To evaluate the CAM methods comprehensively, we use the following popular networks: ResNet-18, ResNet-50, ResNet-101, ResNet-152~\cite{he2016deep}, ResNeXt-50~\cite{xie2017aggregated}, VGG-16~\cite{simonyan2014very}, EfficientNet-B0~\cite{tan2019efficientnet}, MobileNet-V2~\cite{sandler2018mobilenetv2}, and Swin Transformer models in Tiny, Small, Base, and Large configurations~\cite{liu2021Swin}. All network weights (IMAGENET1K\_V1) were obtained directly from PyTorch and Timm \cite{rw2019timm}. The top-1 accuracies of these networks on ILSVRC2012 are appended to their names in Tables \ref{tab:results1_other}, \ref{tab:results1_swin}, \ref{tab:results2_others} and \ref{tab:results2_swin}.

\subsubsection{Compared CAM methods}
We compare the state-of-the-art gradient-based CAM methods in our evaluation: GradCAM \cite{selvaraju2017grad}, HiResCAM \cite{hires2020rachel}, GradCAM-E \cite{jacobgilpytorchcam}, LayerCAM \cite{jiang2021layercam}, XGradCAM \cite{fu2020axiom}, GradCAM++ \cite{chattopadhay2018grad}, RandomCAM \cite{jacobgilpytorchcam}, and the proposed ShapleyCAM, ShapleyCAM-H, ShapleyCAM-E. RandomCAM serves as a baseline, producing a random uniform scalar for each activation map in the range of \([-1, 1]\) to serve as the weight for generating the heatmap. We exclude gradient-free methods due to their long run times \cite{minh2023overview, torcham2020}. For example, ScoreCAM takes over 32 hours with ResNet-50 on ILSVRC2012, while ShapleyCAM and GradCAM finish in 40 and 30 minutes, respectively.

\begin{figure}[htbp]
    \centering
    \includegraphics[width=0.4\textwidth]{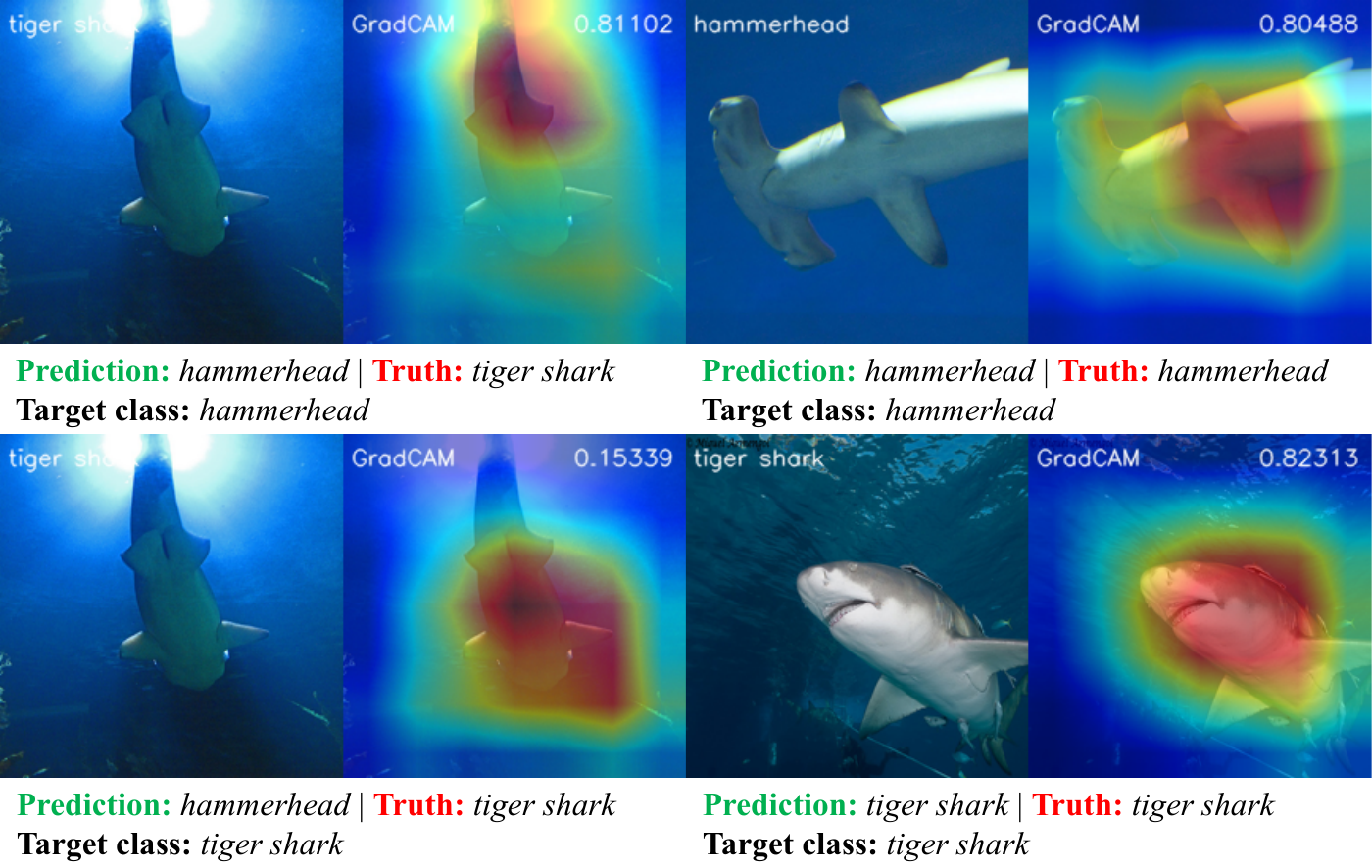} 
    \caption{GradCAM on ResNet-18 using ReST with the layer preceding GAP as the
target layer. Left: ResNet-18 incorrectly predicts a \textit{tiger shark} as a \textit{hammerhead}. Right: ResNet-18 accurately predicts the \textit{tiger shark} and the \textit{hammerhead}}
    \label{fig:predict_label}
\end{figure}

\subsubsection{Target class}

Previous studies often use the true label as the target class \cite{hires2020rachel,fu2020axiom} or focus on correctly predicted images \cite{chattopadhay2018grad,jiang2021layercam}. These settings are reasonable when a well-trained network consistently makes accurate predictions, and the explainer’s performance is measured by identifying corresponding evidence in the input image. However, we also seek insights when the model's predictions are incorrect. The setting of using true label actually introduces the confounding factor of model accuracy when evaluating the precision of explainability. The explainability methods should focus on uncovering the model’s decision-making process, helping users lift the veil on its inference for both correct and incorrect predictions. 

As illustrated in the top-left of Fig. \ref{fig:predict_label}, when ResNet-18 incorrectly predicts a \textit{tiger shark} as a \textit{hammerhead}, the corresponding heatmap highlights the tail of the \textit{tiger shark}, rather than the entire body as shown in the right column of Fig. \ref{fig:predict_label}. In such cases, users are understandably skeptical of the correctness of the prediction and may infer that the model's error is due to the similarity between the \textit{tiger shark}'s tail and a \textit{hammerhead}. Moreover, we often only have access to the model’s predictions in practice. Thus, we use the \textbf{predicted label} as the target class in our experiments, unless specified otherwise.

\begin{table*}[htbp]
\setlength{\tabcolsep}{.25em}
\renewcommand{\arraystretch}{1.05}
\centering
\caption{Evaluation of different CAM methods on eight different CNN backbones with the layer preceding the GAP layer as the target layer}
\label{tab:results1_other}
\resizebox{0.7\textwidth}{!}{
\begin{tabular}{l cccccc cc cccccc}
\hline
& \multicolumn{6}{c}{\textbf{ResNet-18 (69.76\%)}} & & \multicolumn{6}{c}{\textbf{ResNet-50 (76.13\%)}} \\
\cmidrule{2-7} \cmidrule{9-14}
\textbf{Method} & AD $\downarrow$ & Coh $\uparrow$ & Com $\downarrow$ & \textbf{ADCC} $\uparrow$ & \textbf{IC} $\uparrow$ & \textbf{ADD} $\uparrow$ & & AD $\downarrow$ & Coh $\uparrow$ & Com $\downarrow$ & \textbf{ADCC} $\uparrow$ & \quad\textbf{IC} $\uparrow$ & \textbf{ADD} $\uparrow$ \\
\hline
RandomCAM & 71.15 & 52.80 & 20.24 & 45.36 & 6.64 & 19.55 & & 67.35 & 54.29 & 19.75 & 48.77 & 8.24 & 14.36 \\
GradCAM++ & 19.86 & 98.55 & 41.84 & 75.34 & 31.30 & 51.52 & & 14.27 & 98.03 & 41.16 & 77.20 & 38.19 & 43.60 \\
\cdashline{1-14}
GradCAM-E & 19.84 & \underline{98.58} & 41.96 & 75.29 & 31.26 & 51.60 & & 14.26 & \underline{98.06} & 41.74 & 76.87 & 37.95 & 43.59 \\
ShapleyCAM-E & 19.81 & \textbf{98.58} & 41.93 & 75.31 & 31.29 & 51.63 & & 14.24 & \textbf{98.07} & 41.70 & 76.90 & 38.04 & 43.63 \\
\cdashline{1-14}
GradCAM & \underline{18.73} & 97.51 & \underline{38.84} & \underline{77.10} & \underline{33.77} & \underline{51.89} & & \textbf{13.84} & 97.07 & \underline{38.42} & \underline{78.64} & \underline{40.05} & \underline{43.82} \\
ShapleyCAM & \textbf{18.71} & 97.65 & \textbf{38.68} & \textbf{77.22} & \textbf{33.92} & \textbf{51.90} & & \underline{13.85} & 97.19 & \textbf{38.23} & \textbf{78.77} & \textbf{40.17} & \textbf{43.84} \\
\hline
& \multicolumn{6}{c}{\textbf{ResNet-101 (77.38\%)}} & & \multicolumn{6}{c}{\textbf{ResNet-152 (78.32\%)}} \\
\cmidrule{2-7} \cmidrule{9-14}
\textbf{Method} & AD $\downarrow$ & Coh $\uparrow$ & Com $\downarrow$ & \textbf{ADCC} $\uparrow$ & \textbf{IC} $\uparrow$ & \textbf{ADD} $\uparrow$ & & AD $\downarrow$ & Coh $\uparrow$ & Com $\downarrow$ & \textbf{ADCC} $\uparrow$ & \quad\textbf{IC} $\uparrow$ & \textbf{ADD} $\uparrow$ \\
\hline
RandomCAM & 65.07 & 55.09 & 19.79 & 50.63 & 9.56 & 14.43 & & 64.14 & 55.18 & 19.64 & 51.33 & 9.84 & 13.44 \\
GradCAM++ & 13.17 & 98.46 & 41.04 & 77.65 & 39.63 & 42.16 & & 12.62 & 98.49 & 40.71 & 78.00 & 39.06 & 39.84 \\
\cdashline{1-14}
GradCAM-E & 13.18 & \underline{98.51} & 41.64 & 77.31 & 39.37 & 42.21 & & 12.59 & \underline{98.53} & 41.32 & 77.66 & 38.78 & 39.91 \\
ShapleyCAM-E & 13.15 & \textbf{98.51} & 41.60 & 77.34 & 39.45 & 42.25 & & 12.57 & \textbf{98.53} & 41.28 & 77.69 & 38.88 & 39.96 \\
\cdashline{1-14}
GradCAM & \underline{12.83} & 97.40 & \underline{38.43} & \underline{78.98} & \underline{41.65} & \underline{42.50} & & \underline{12.44} & 97.52 & \underline{38.22} & \underline{79.24} & \underline{40.90} & \underline{40.19} \\
ShapleyCAM & \textbf{12.82} & 97.50 & \textbf{38.24} & \textbf{79.11} & \textbf{41.77} & \textbf{42.54} & & \textbf{12.43} & 97.62 & \textbf{38.03} & \textbf{79.36} & \textbf{41.14} & \textbf{40.23} \\
\hline
& \multicolumn{6}{c}{\textbf{ResNeXt-50 (77.62\%)}} & & \multicolumn{6}{c}{\textbf{MobileNet-V2 (71.88\%)}} \\
\cmidrule{2-7} \cmidrule{9-14}
\textbf{Method} & AD $\downarrow$ & Coh $\uparrow$ & Com $\downarrow$ & \textbf{ADCC} $\uparrow$ & \textbf{IC} $\uparrow$ & \textbf{ADD} $\uparrow$ & & AD $\downarrow$ & Coh $\uparrow$ & Com $\downarrow$ & \textbf{ADCC} $\uparrow$ & \quad\textbf{IC} $\uparrow$ & \textbf{ADD} $\uparrow$ \\
\hline
RandomCAM & 63.93 & 51.73 & 20.93 & 50.25 & 9.76 & 13.85 & & 71.37 & 56.43 & 19.40 & 46.12 & 6.00 & 17.92 \\
GradCAM++ & 12.54 & 97.83 & 45.81 & 74.79 & 38.96 & 40.00 & & 18.92 & 98.58 & 44.00 & 74.38 & 30.45 & 51.53 \\
\cdashline{1-14}
GradCAM-E & 12.59 & \underline{97.87} & 46.33 & 74.46 & 38.71 & 39.92 & & 18.96 & \textbf{98.60} & 44.30 & 74.19 & 30.28 & \textbf{51.55} \\
ShapleyCAM-E & 12.57 & \textbf{97.87} & 46.30 & 74.49 & 38.77 & 39.95 & & 18.97 & \underline{98.60} & 44.24 & 74.23 & 30.34 & \underline{51.54} \\
\cdashline{1-14}
GradCAM & \underline{12.19} & 96.89 & \underline{43.06} & \underline{76.39} & \underline{40.38} & \underline{40.47} & & \textbf{18.27} & 97.54 & \underline{40.23} & \underline{76.49} & \underline{32.54} & 51.06 \\
ShapleyCAM & \textbf{12.18} & 97.01 & \textbf{42.88} & \textbf{76.52} & \textbf{40.51} & \textbf{40.51} & & \underline{18.31} & 97.65 & \textbf{39.99} & \textbf{76.64} & \textbf{32.61} & 51.01 \\
\hline
& \multicolumn{6}{c}{\textbf{VGG-16 (71.59\%)}} & & \multicolumn{6}{c}{\textbf{EfficientNet-B0 (77.69\%)}} \\
\cmidrule{2-7} \cmidrule{9-14}
\textbf{Method} & AD $\downarrow$ & Coh $\uparrow$ & Com $\downarrow$ & \textbf{ADCC} $\uparrow$ & \textbf{IC} $\uparrow$ & \textbf{ADD} $\uparrow$ & & AD $\downarrow$ & Coh $\uparrow$ & Com $\downarrow$ & \textbf{ADCC} $\uparrow$ & \quad\textbf{IC} $\uparrow$ & \textbf{ADD} $\uparrow$ \\
\hline
RandomCAM & 72.39 & 60.94 & 16.84 & 46.40 & 4.64 & 15.07 & & 69.13 & 61.76 & 16.70 & 49.52 & 8.42 & 16.00 \\
GradCAM++ & 22.44 & \textbf{96.23} & 31.90 & 79.01 & 26.65 & 38.00 & & 30.42 & 97.35 & \textbf{23.21} & 79.65 & 26.98 & 34.66 \\
XGradCAM & \textbf{20.42} & 90.82 & 30.98 & 78.81 & \textbf{32.51} & \textbf{40.24} & & - & - & - & - & - & - \\
LayerCAM & - & - & - & - & - & - & & 30.58 & 97.37 & \underline{23.23} & 79.57 & 26.79 & 34.56 \\
\cdashline{1-14}
GradCAM-E & 22.47 & 96.15 & 30.51 & \underline{79.60} & 27.18 & 37.65 & & 29.21 & \underline{97.49} & 24.99 & 79.54 & 27.98 & 34.25 \\
ShapleyCAM-E & \underline{22.41} & \underline{96.19} & 30.45 & \textbf{79.65} & 27.30 & 37.75 & & 29.23 & \textbf{97.50} & 24.98 & 79.54 & 27.97 & 34.24 \\
\cdashline{1-14}
HiResCAM & 25.07 & 90.44 & \underline{27.90} & 78.39 & 27.72 & 35.66 & & - & - & - & - & - & - \\
ShapleyCAM-H & 25.01 & 90.56 & \textbf{27.76} & 78.49 & 27.85 & 35.82 & & - & - & - & - & - & - \\
\cdashline{1-14}
GradCAM & 22.67 & 89.58 & 30.06 & 78.14 & 29.99 & 38.68 & & \textbf{24.95} & 96.44 & 25.98 & \underline{80.64} & \textbf{32.44} & \textbf{37.90} \\
ShapleyCAM & 22.68 & 89.69 & 29.86 & 78.25 & \underline{30.12} & \underline{38.76} & & \underline{25.02} & 96.89 & 26.01 & \textbf{80.70} & \underline{32.30} & \underline{37.83} \\
\hline
\end{tabular}
}
\end{table*}

\subsection{Evaluation Metrics for Explanations}

To comprehensively evaluate CAM-based explanations, we employ a set of metrics \cite{chattopadhay2018grad, poppi2021ADCC, liftcam2021jung} to measure different aspects of explanation quality. All these metrics quantify explanation performance based on changes in prediction confidence when regions highlighted by CAMs are focused on or masked, as well as by the visual quality of the explanations. Notably, we exclude certain localization metrics, such as Intersection over Union (IoU), as they focus on localization performance rather than explainability, as discussed in Section \ref{subsec:exp}. 

\begin{figure}[htbp]
    \centering
    \includegraphics[width=0.4\textwidth]{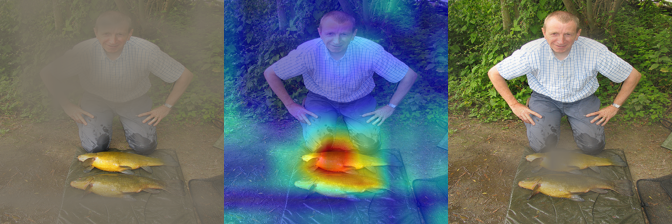} 
    \caption{Explanation map, visual explanation, and anti-explanation map generated by ShapleyCAM using the last convolutional layer of VGG-16}
    \label{fig:metric_visual}
\vspace{-0.2cm}
\end{figure}
To facilitate understanding of these metrics, we define some terminology: after a CAM method produces the normalized and upsampled heatmap \(\text{H}_c(x)\), it is linearly combined with the raw image \(x\) to obtain the \textit{visual explanation}. The \textit{explanation map} \(x \odot \text{H}_c(x)\) represents the Hadamard product of the heatmap and the raw image, masking unimportant pixels. Masking the important pixels yields the \textit{anti-explanation map} \(x \odot (1-\text{H}_c(x))\). An illustration of the \textit{explanation map}, \textit{visual explanation}, and \textit{anti-explanation map} is provided in Fig. \ref{fig:metric_visual}.

\begin{itemize}
    \item \textbf{Average Drop (AD) \cite{chattopadhay2018grad}} quantifies the average decline in model confidence for the target class when using only the explanation map instead of the full input image:
    \begin{equation}
        \text{AD} = \frac{1}{N} \sum_{i=1}^{N} \frac{\max(0, y_i^c - o_i^c)}{y_i^c} \times 100.
    \end{equation}
    where \( y_i^c \) and \( o_i^c \) are the model's post-softmax scores for class \( c \) with the full image and explanation map, respectively. A lower AD indicates that the explanation map captures the most relevant features, whose presence helps maintain the model's confidence.

    \item \textbf{Coherency (Coh) \cite{poppi2021ADCC}} assesses the consistency of the CAM method by calculating the normalized Pearson Correlation Coefficient between the original heatmap and the heatmap generated from the explanation map:
    \begin{equation}
        \text{Coh}(x) = \frac{1}{2}\frac{\mathrm{Cov}(\text{H}_c(x \odot \text{H}_c(x)), \text{H}_c(x))}{\sigma_{\text{H}_c(x \odot \text{H}_c(x))} \sigma_{\text{H}_c(x)}}+\frac{1}{2}.
    \end{equation}
    \item \textbf{Complexity (Com) \cite{poppi2021ADCC}} evaluates the simplicity of the explanation map by calculating the L1 norm:
    \begin{equation}
        \text{Complexity}(x) = \| \text{H}_c(x) \|_1.
    \end{equation}
    \item \textbf{Average DCC (ADCC) \cite{poppi2021ADCC}} is the harmonic mean of Average Drop, Coherency (averaged across all samples), and Complexity (averaged across all samples) to provide an overall measure of explanation quality: 
    \begin{equation}
        \text{ADCC}(x) = \frac{3} {\frac{1}{\text{Coh}(x)} + \frac{1}{1 - \text{Com}(x)} + \frac{1}{1 - \text{AD}(x)}}.
    \end{equation}
   A higher ADCC score reflects a balanced trade-off among preserving model confidence, providing consistent explanations, and ensuring simplicity.
    
    \item \textbf{Increase in Confidence (IC) \cite{chattopadhay2018grad}} measures the proportion of instances in which the model's confidence increases when using the explanation map instead of the full image, where \( \mathbb{I} \) is the indicator function:
    \begin{equation}
        \text{IC} = \frac{1}{N} \sum_{i=1}^{N} \mathbb{I}(y_i^c < o_i^c) \times 100.
    \end{equation}

    \item \textbf{Average Drop in Deletion (ADD) \cite{liftcam2021jung}} assesses the drop in confidence when using anti-explanation map:
    \begin{equation}
        \text{ADD} = \frac{1}{N} \sum_{i=1}^{N} \frac{\max(0, y_i^c - d_i^c)}{y_i^c} \times 100.
    \end{equation}
    Here, \( d_i^c \) is the model's post-softmax score for class \( c \) with the anti-explanation map. Higher ADD scores indicate effective identification of critical features, as removing these pixels leads to a significant decrease in confidence.
\end{itemize}

\begin{table*}[htbp]
\setlength{\tabcolsep}{.25em}
\renewcommand{\arraystretch}{1.05}
\centering
\caption{Evaluation of different CAM methods on Swin Transformers with the layer preceding the GAP layer as the target layer}
\label{tab:results1_swin}
\resizebox{0.7\textwidth}{!}{
\begin{tabular}{l cccccc cc cccccc}
\hline
& \multicolumn{6}{c}{\textbf{Swin-T (80.91\%)}} & & \multicolumn{6}{c}{\textbf{Swin-S (83.05\%)}} \\
\cmidrule{2-7} \cmidrule{9-14}
\textbf{Method} & AD $\downarrow$ & Coh $\uparrow$ & Com $\downarrow$ & \textbf{ADCC} $\uparrow$ & \textbf{IC} $\uparrow$ & \textbf{ADD} $\uparrow$ & & AD $\downarrow$ & Coh $\uparrow$ & Com $\downarrow$ & \textbf{ADCC} $\uparrow$ & \quad\textbf{IC} $\uparrow$ & \textbf{ADD} $\uparrow$ \\
\hline
RandomCAM & 73.23 & 60.03 & 16.93 & 45.42 & 2.66 & 15.27 & & 68.56 & 57.45 & 17.38 & 48.93 & 6.90 & 15.84 \\
GradCAM++ & 32.34 & 90.71 & 30.31 & 74.71 & 11.62 & \underline{39.28} & & 26.95 & 89.32 & \textbf{26.70} & \underline{77.87} & 26.77 & 31.46 \\
LayerCAM & 32.23 & 90.88 & 30.64 & 74.67 & 11.50 & \textbf{39.58} & & 26.89 & \textbf{89.65} & \underline{26.91} & \textbf{77.89} & 26.62 & 31.63 \\
\cdashline{1-14}
GradCAM-E & 32.52 & 88.31 & 33.13 & 73.00 & 10.12 & 39.00 & & 27.12 & 89.29 & 32.05 & 75.69 & 23.20 & \textbf{35.07} \\
ShapleyCAM-E & 32.52 & 88.50 & 33.11 & 73.05 & 10.15 & 38.95 & & 27.12 & \underline{89.43} & 32.00 & 75.74 & 23.24 & \underline{35.00} \\
\cdashline{1-14}
GradCAM & \textbf{31.89} & \underline{91.50} & \textbf{30.13} & \underline{75.14} & \textbf{11.80} & 39.11 & & \underline{26.28} & 88.97 & 28.12 & 77.49 & \textbf{26.95} & 32.65 \\
ShapleyCAM & \underline{31.92} & \textbf{91.88} & \underline{30.19} & \textbf{75.19} & \underline{11.69} & 39.06 & & \textbf{26.28} & 89.34 & 28.17 & 77.56 & \underline{26.93} & 32.62 \\
\hline
& \multicolumn{6}{c}{\textbf{Swin-B (84.71\%)}} & & \multicolumn{6}{c}{\textbf{Swin-L (85.83\%)}} \\
\cmidrule{2-7} \cmidrule{9-14}
\textbf{Method} & AD $\downarrow$ & Coh $\uparrow$ & Com $\downarrow$ & \textbf{ADCC} $\uparrow$ & \textbf{IC} $\uparrow$ & \textbf{ADD} $\uparrow$ & & AD $\downarrow$ & Coh $\uparrow$ & Com $\downarrow$ & \textbf{ADCC} $\uparrow$ & \quad\textbf{IC} $\uparrow$ & \textbf{ADD} $\uparrow$ \\
\hline
RandomCAM & 70.71 & 56.69 & 17.08 & 46.99 & 6.29 & 16.16 & & 70.95 & 57.30 & 16.47 & 46.99 & 3.47 & 14.04 \\
GradCAM++ & 27.68 & 85.99 & \textbf{28.96} & 75.88 & \underline{26.60} & 34.27 & & \underline{33.31} & \underline{88.00} & \textbf{25.44} & \underline{75.43} & 13.95 & 28.31 \\
LayerCAM & 27.57 & 86.18 & 29.24 & 75.87 & \textbf{26.63} & 34.54 & & \textbf{33.20} & \textbf{88.24} & \underline{25.64} & \textbf{75.47} & 13.82 & 28.49 \\
\cdashline{1-14}
GradCAM-E & 28.93 & 85.59 & 33.54 & 73.53 & 21.96 & \textbf{37.59} & & 35.55 & 87.80 & 30.57 & 72.63 & 10.49 & \textbf{32.20} \\
ShapleyCAM-E & 28.93 & 85.76 & 33.50 & 73.59 & 22.05 & \underline{37.54} & & 35.56 & 87.93 & 30.52 & 72.67 & 10.50 & \underline{32.14} \\
\cdashline{1-14}
GradCAM & \underline{27.52} & \underline{86.73} & \underline{29.19} & \underline{76.05} & 26.28 & 34.62 & & 33.91 & 87.37 & 26.06 & 74.81 & \textbf{14.07} & 29.06 \\
ShapleyCAM & \textbf{27.52} & \textbf{87.05} & 29.20 & \textbf{76.12} & 26.20 & 34.55 & & 33.91 & 87.68 & 26.05 & 74.89 & \underline{14.01} & 28.99 \\
\hline
\end{tabular}
}
\vspace{-0.1cm}
\end{table*}

These metrics aim to characterize various aspects of explanation quality. However, they may not fully capture the explainability of CAMs. For instance, these metrics can be misleading. As demonstrated by \cite{poppi2021ADCC}, a simplistic method like ``Fake-CAM", which assigns equal weights to nearly all pixels except for one, can attain nearly perfect scores for IC, AD, and Coh without delivering meaningful explanations. Therefore, we advocate for further research to develop more accurate and comprehensive evaluation methods in this domain.

In our experiments, we primarily focus on ADCC, IC, and ADD as the measures for evaluating CAM methods.

\subsection{Quantitative Comparison of CAM Methods Using the Layer Preceding GAP as the Target Layer}
\label{subsec:expriments}

Although the original CAM \cite{zhou2016learning} is often criticized for its heavy reliance on the GAP layer, this layer has become a fundamental component near the classifier in state-of-the-art networks, including ResNet, ResNeXt, EfficientNet, MobileNet, and even Swin Transformer. VGG-16, however, uses an adaptive average pooling layer that converts the activation maps into \(7 \times 7\) maps. 

Here, we would like to clarify that the layer preceding GAP is distinct from the last convolutional layer. Typically, a normalization layer and a ReLU function follow the last convolutional layer, placing it before the layer preceding GAP (i.e., ReLU layer). When using the layer preceding GAP as the target layer (Table \ref{tab:results1_other} and Table \ref{tab:results1_swin}), many CAM methods reduce to the same approach. However, when using the last convolutional layer for non-Swin Transformer networks (Table \ref{tab:results2_others}) and the first normalization layer of the last transformer block for Swin Transformers (Table \ref{tab:results2_swin}), they perform differently.

We begin by using the layer preceding GAP (i.e., SiLU layer for EfficientNet, LayerNorm layer for Swin Transformers, and ReLU layer for others) as the target layer for all networks except VGG-16. In this case, HiResCAM and XGradCAM are equivalent to GradCAM, as proven in \cite{jiang2021layercam, fu2020axiom}, and ShapleyCAM-H is also equivalent to ShapleyCAM, so we omit the HiResCAM, XGradCAM, and ShapleyCAM-H for brevity except VGG-16. For VGG-16, we select the layer preceding the adaptive averaging pooling layer (i.e., ReLU layer) as the target layer. Since ReLU ensures that the activation maps from all networks except the Swin Transformer and EfficientNet are non-negative, LayerCAM is equivalent to GradCAM-E (see Equations \eqref{eq:gradcam-e} and \eqref{eq:layercam}), so we exclude LayerCAM's results for these networks.

As shown in Table \ref{tab:results1_other}, ShapleyCAM achieves the highest ADCC score in ResNet-18, ResNet-50, ResNet-101, ResNet-152, ResNeXt-50, EfficientNet-B0, and MobileNet-V2. For VGG-16, ShapleyCAM-E has the best ADCC score. The IC and ADD scores for ShapleyCAM are also competitive, ranking first or second in most cases. 

In the case of the Swin Transformer as shown in Table \ref{tab:results1_swin}, ShapleyCAM has the best ADCC scores in both Swin-T and Swin-B, while LayerCAM and GradCAM++ achieve the best and second-best ADCC scores in Swin-S and Swin-L. GradCAM performs best in IC, while GradCAM-E excels in ADD. ShapleyCAM and ShapleyCAM-E also show comparable performance in IC and ADD, respectively.

In summary, when using the layer preceding GAP as the target layer, ShapleyCAM consistently outperforms other compared CAM methods in ADCC, IC, and ADD scores across various CNN architectures, and enjoys competitive performance on Swin Transformer architectures.

\begin{table*}[htbp]
\setlength{\tabcolsep}{.25em}
\renewcommand{\arraystretch}{1.05}
\centering
\caption{Evaluation of different CAM methods on eight different CNN backbones with the last convolutional layer as the target layer}
\label{tab:results2_others}
\resizebox{0.7\textwidth}{!}{
\begin{tabular}{l cccccc cc cccccc}
\hline
& \multicolumn{6}{c}{\textbf{ResNet-18 (69.76\%)}} & & \multicolumn{6}{c}{\textbf{ResNet-50 (76.13\%)}} \\
\cmidrule{2-7} \cmidrule{9-14}
\textbf{Method} & AD $\downarrow$ & Coh $\uparrow$ & Com $\downarrow$ & \textbf{ADCC} $\uparrow$ & \textbf{IC} $\uparrow$ & \textbf{ADD} $\uparrow$ & & AD $\downarrow$ & Coh $\uparrow$ & Com $\downarrow$ & \textbf{ADCC} $\uparrow$ & \quad\textbf{IC} $\uparrow$ & \textbf{ADD} $\uparrow$ \\
\hline
RandomCAM & 72.26 & 50.36 & 21.04 & 43.75 & 5.83 & 19.48 & & 67.85 & 52.75 & 20.28 & 47.92 & 7.52 & 14.26 \\
GradCAM++ & 28.75 & 97.55 & \underline{33.14} & 76.45 & 23.98 & 46.96 & & 21.52 & 96.07 & \underline{33.76} & 78.43 & 30.92 & 35.80 \\
XGradCAM & 48.36 & 67.52 & \textbf{28.19} & 62.37 & 13.18 & 31.47 & & 40.03 & 68.71 & \textbf{27.76} & 66.56 & 18.64 & 24.07 \\
LayerCAM & 23.33 & 98.27 & 38.52 & 75.98 & 27.61 & 48.70 & & 19.18 & 96.42 & 37.02 & 77.67 & 32.70 & 36.29 \\
\cdashline{1-14}
GradCAM-E & 23.26 & \underline{98.29} & 38.25 & 76.15 & 27.79 & 48.04 & & 18.51 & \underline{96.69} & 37.46 & 77.72 & 33.39 & 36.46 \\
ShapleyCAM-E & 23.27 & \textbf{98.30} & 38.19 & 76.17 & 27.81 & 48.04 & & 18.52 & \textbf{96.70} & 37.41 & 77.73 & 33.41 & 36.48 \\
\cdashline{1-14}
HiResCAM & 21.47 & 97.25 & 37.04 & \underline{77.12} & 30.48 & 49.39 & & 18.00 & 95.49 & 35.21 & \underline{78.74} & 34.80 & 37.35 \\
ShapleyCAM-H & 21.49 & 97.41 & 36.92 & \textbf{77.21} & 30.50 & 49.37 & & 18.06 & 95.69 & 35.12 & \textbf{78.80} & 34.85 & 37.36 \\
\cdashline{1-14}
GradCAM & \textbf{20.03} & 97.41 & 38.36 & 76.94 & \underline{31.92} & \textbf{51.07} & & \textbf{16.91} & 95.63 & 36.36 & 78.53 & \underline{36.21} & \underline{38.37} \\
ShapleyCAM & \underline{20.04} & 97.57 & 38.23 & 77.03 & \textbf{31.94} & \underline{51.07} & & \underline{16.98} & 95.81 & 36.26 & 78.59 & \textbf{36.25} & \textbf{38.38} \\
\hline
& \multicolumn{6}{c}{\textbf{ResNet-101 (77.38\%)}} & & \multicolumn{6}{c}{\textbf{ResNet-152 (78.32\%)}} \\
\cmidrule{2-7} \cmidrule{9-14}
\textbf{Method} & AD $\downarrow$ & Coh $\uparrow$ & Com $\downarrow$ & \textbf{ADCC} $\uparrow$ & \textbf{IC} $\uparrow$ & \textbf{ADD} $\uparrow$ & & AD $\downarrow$ & Coh $\uparrow$ & Com $\downarrow$ & \textbf{ADCC} $\uparrow$ & \quad\textbf{IC} $\uparrow$ & \textbf{ADD} $\uparrow$ \\
\hline
RandomCAM & 65.38 & 53.66 & 20.46 & 49.92 & 9.09 & 14.09 & & 64.38 & 51.87 & 20.86 & 50.01 & 9.57 & 13.16 \\
GradCAM++ & 17.19 & 96.32 & 37.83 & 77.83 & 35.93 & 36.05 & & 17.96 & 96.10 & \underline{35.76} & \textbf{78.61} & 34.40 & 33.59 \\
XGradCAM & 35.35 & 70.29 & \textbf{29.70} & 68.31 & 22.48 & 24.55 & & 33.37 & 70.58 & \textbf{30.43} & 68.89 & 23.18 & 23.62 \\
LayerCAM & 16.15 & 96.42 & 40.31 & 76.82 & 36.12 & 35.98 & & 16.22 & 96.32 & 38.93 & 77.53 & 35.40 & 33.76 \\
\cdashline{1-14}
GradCAM-E & 15.73 & \underline{96.86} & 40.40 & 76.98 & 36.36 & 36.30 & & 15.09 & \underline{97.03} & 40.04 & 77.40 & 36.27 & 34.38 \\
ShapleyCAM-E & 15.74 & \textbf{96.87} & 40.37 & 77.00 & 36.40 & 36.32 & & 15.11 & \textbf{97.04} & 40.01 & 77.41 & 36.28 & 34.40 \\
\cdashline{1-14}
HiResCAM & 15.45 & 95.28 & 37.78 & \underline{78.14} & 38.22 & 36.89 & & 15.02 & 95.51 & 37.78 & 78.31 & 37.75 & 35.01 \\
ShapleyCAM-H & 15.50 & 95.46 & \underline{37.70} & \textbf{78.21} & 38.26 & 36.91 & & 15.04 & 95.67 & 37.72 & \underline{78.37} & 37.78 & 35.04 \\
\cdashline{1-14}
GradCAM & \textbf{14.59} & 95.46 & 38.97 & 77.78 & \underline{39.38} & \underline{37.56} & & \textbf{14.11} & 95.72 & 38.86 & 78.03 & \underline{39.09} & \underline{35.99} \\
ShapleyCAM & \underline{14.63} & 95.63 & 38.89 & 77.85 & \textbf{39.40} & \textbf{37.58} & & \underline{14.14} & 95.87 & 38.79 & 78.09 & \textbf{39.11} & \textbf{36.02} \\
\hline
& \multicolumn{6}{c}{\textbf{ResNeXt-50 (77.62\%)}} & & \multicolumn{6}{c}{\textbf{MobileNet-V2 (71.88\%)}} \\
\cmidrule{2-7} \cmidrule{9-14}
\textbf{Method} & AD $\downarrow$ & Coh $\uparrow$ & Com $\downarrow$ & \textbf{ADCC} $\uparrow$ & \textbf{IC} $\uparrow$ & \textbf{ADD} $\uparrow$ & & AD $\downarrow$ & Coh $\uparrow$ & Com $\downarrow$ & \textbf{ADCC} $\uparrow$ & \quad\textbf{IC} $\uparrow$ & \textbf{ADD} $\uparrow$ \\
\hline
RandomCAM & 64.53 & 47.87 & 21.57 & 48.52 & 9.16 & 13.87 & & 72.05 & 54.25 & 20.44 & 44.93 & 5.44 & 18.03 \\
GradCAM++ & 18.97 & 95.99 & \underline{35.70} & \textbf{78.31} & 32.34 & 34.30 & & 19.12 & 98.40 & 43.62 & 74.51 & 30.45 & 51.40 \\
XGradCAM & 31.89 & 70.94 & \textbf{32.68} & 68.75 & 23.51 & 25.54 & & 55.12 & 63.76 & \textbf{26.07} & 58.26 & 9.27 & 25.13 \\
LayerCAM & 16.61 & 96.79 & 39.50 & 77.22 & 33.98 & 35.19 & & 16.41 & \underline{98.66} & 49.23 & 71.78 & 33.43 & \underline{58.81} \\
\cdashline{1-14}
GradCAM-E & 15.88 & \underline{96.87} & 40.48 & 76.90 & 34.78 & 34.90 & & 16.41 & \textbf{98.66} & 49.23 & 71.78 & 33.43 & \textbf{58.81} \\
ShapleyCAM-E & 15.89 & \textbf{96.88} & 40.42 & 76.93 & 34.84 & 34.91 & & 16.43 & 98.66 & 49.15 & 71.83 & 33.47 & 58.76 \\
\cdashline{1-14}
HiResCAM & 14.60 & 96.09 & 40.03 & 77.33 & 37.08 & 36.66 & & \textbf{15.90} & 97.35 & 44.19 & 74.85 & \underline{36.10} & 55.59 \\
ShapleyCAM-H & 14.60 & 96.24 & 39.92 & \underline{77.43} & 37.10 & 36.69 & & \underline{15.94} & 97.44 & 43.91 & 75.02 & \textbf{36.19} & 55.44 \\
\cdashline{1-14}
GradCAM & \underline{13.69} & 96.25 & 40.75 & 77.21 & \textbf{38.34} & \underline{37.93} & & 16.97 & 97.69 & 42.15 & \underline{75.83} & 33.84 & 51.55 \\
ShapleyCAM & \textbf{13.68} & 96.39 & 40.63 & 77.31 & \underline{38.32} & \textbf{37.96} & & 16.96 & 97.80 & \underline{41.94} & \textbf{75.97} & 34.03 & 51.54 \\
\hline
& \multicolumn{6}{c}{\textbf{VGG-16 (71.59\%)}} & & \multicolumn{6}{c}{\textbf{EfficientNet-B0 (77.69\%)}} \\
\cmidrule{2-7} \cmidrule{9-14}
\textbf{Method} & AD $\downarrow$ & Coh $\uparrow$ & Com $\downarrow$ & \textbf{ADCC} $\uparrow$ & \textbf{IC} $\uparrow$ & \textbf{ADD} $\uparrow$ & & AD $\downarrow$ & Coh $\uparrow$ & Com $\downarrow$ & \textbf{ADCC} $\uparrow$ & \quad\textbf{IC} $\uparrow$ & \textbf{ADD} $\uparrow$ \\
\hline
RandomCAM & 74.95 & 50.79 & 16.42 & 41.91 & 4.93 & 14.38 & & 67.14 & 58.09 & 19.31 & 49.96 & 9.56 & 16.76 \\
GradCAM++ & 67.26 & 61.72 & 21.52 & 50.43 & 5.91 & 17.44 & & \underline{16.91} & 94.88 & 42.04 & 75.32 & 37.78 & 43.68 \\
XGradCAM & 77.24 & 61.09 & 20.61 & 41.15 & 2.94 & 13.94 & & 52.40 & 63.51 & \textbf{25.12} & 59.87 & 15.65 & 22.33 \\
LayerCAM & 57.16 & 87.47 & 16.78 & 64.11 & 8.04 & 24.48 & & 21.58 & 97.99 & 32.87 & 79.25 & 33.32 & 40.00 \\
\cdashline{1-14}
GradCAM-E & 57.16 & 87.47 & 16.78 & 64.11 & 8.04 & 24.48 & & 21.29 & \textbf{98.03} & 33.24 & 79.19 & 33.49 & 40.28 \\
ShapleyCAM-E & \underline{57.14} & \textbf{87.57} & 16.74 & \underline{64.16} & \underline{8.09} & \underline{24.56} & & 23.02 & \underline{97.99} & 31.06 & 79.57 & 32.00 & 39.11 \\
\cdashline{1-14}
HiResCAM & 73.26 & 77.46 & \underline{12.34} & 48.62 & 3.86 & 23.13 & & 19.25 & 96.73 & 31.63 & \underline{80.32} & 37.49 & 43.15 \\
ShapleyCAM-H & 73.27 & 77.91 & \textbf{12.29} & 48.66 & 3.84 & 23.23 & & 21.14 & 97.05 & \underline{29.45} & \textbf{80.74} & 35.56 & 41.58 \\
\cdashline{1-14}
GradCAM & 76.68 & 69.28 & 12.92 & 43.60 & 5.43 & 22.98 & & \textbf{16.78} & 95.76 & 38.38 & 77.54 & \textbf{39.00} & \textbf{44.50} \\
ShapleyCAM & \textbf{31.30} & \underline{87.52} & 25.32 & \textbf{76.20} & \textbf{22.39} & \textbf{35.94} & & 17.12 & 95.10 & 40.41 & 76.22 & \underline{38.23} & \underline{44.32} \\
\hline
\end{tabular}
}
\end{table*}

\begin{table*}[htbp]
\setlength{\tabcolsep}{.25em}
\renewcommand{\arraystretch}{1.05}
\centering
\caption{Evaluation of different CAM methods on Swin Transformers with the first normalization layer of the last transformer block as the target layer}
\label{tab:results2_swin}
\resizebox{0.7\textwidth}{!}{
\begin{tabular}{l cccccc cc cccccc}
\hline
& \multicolumn{6}{c}{\textbf{Swin-T (80.91\%)}} & & \multicolumn{6}{c}{\textbf{Swin-S (83.05\%)}} \\
\cmidrule{2-7} \cmidrule{9-14}
\textbf{Method} & AD $\downarrow$ & Coh $\uparrow$ & Com $\downarrow$ & \textbf{ADCC} $\uparrow$ & \textbf{IC} $\uparrow$ & \textbf{ADD} $\uparrow$ & & AD $\downarrow$ & Coh $\uparrow$ & Com $\downarrow$ & \textbf{ADCC} $\uparrow$ & \quad\textbf{IC} $\uparrow$ & \textbf{ADD} $\uparrow$ \\
\hline
RandomCAM & 70.38 & 54.73 & 19.49 & 46.54 & 3.12 & 16.60 & & 66.28 & 53.44 & 20.58 & 49.21 & 6.54 & 17.74 \\
GradCAM++ & 86.45 & 42.36 & 14.48 & 27.50 & 1.32 & 10.36 & & 80.49 & 56.47 & 17.95 & 36.97 & 3.39 & 12.00 \\
XGradCAM & 76.20 & 54.98 & 16.27 & 41.58 & 2.80 & 13.79 & & 74.16 & 53.28 & 17.02 & 43.16 & 5.22 & 14.57 \\
LayerCAM & 89.14 & 64.69 & \textbf{7.79} & 25.34 & 1.52 & 11.44 & & 89.86 & 61.23 & \underline{10.90} & 23.78 & 1.47 & 10.73 \\
\cdashline{1-14}
GradCAM-E & \underline{52.51} & \textbf{84.61} & 21.63 & \textbf{65.74} & \underline{4.75} & \textbf{25.73} & & \textbf{45.15} & \textbf{82.74} & 23.42 & \textbf{69.16} & \underline{14.49} & \textbf{26.40} \\
ShapleyCAM-E & \textbf{52.28} & \underline{83.97} & 21.87 & \underline{65.70} & \textbf{4.95} & \underline{25.24} & & \underline{46.50} & \underline{82.35} & 21.66 & \underline{68.81} & \textbf{15.11} & \underline{24.60} \\
\cdashline{1-14}
HiResCAM & 89.56 & 66.46 & \underline{9.01} & 24.62 & 1.39 & 12.42 & & 90.59 & 63.06 & \textbf{8.92} & 22.54 & 1.44 & 10.63 \\
ShapleyCAM-H & 76.08 & 61.83 & 14.29 & 43.08 & 2.48 & 15.30 & & 79.52 & 64.06 & 13.46 & 39.48 & 3.82 & 14.98 \\
\cdashline{1-14}
GradCAM & 85.97 & 49.57 & 16.65 & 28.99 & 1.30 & 10.58 & & 82.00 & 57.43 & 19.02 & 35.16 & 3.01 & 12.43 \\
ShapleyCAM & 74.69 & 49.73 & 29.73 & 40.63 & 2.96 & 17.15 & & 64.84 & 56.33 & 33.33 & 49.03 & 7.22 & 17.13 \\
\hline
& \multicolumn{6}{c}{\textbf{Swin-B (84.71\%)}} & & \multicolumn{6}{c}{\textbf{Swin-L (85.83\%)}} \\
\cmidrule{2-7} \cmidrule{9-14}
\textbf{Method} & AD $\downarrow$ & Coh $\uparrow$ & Com $\downarrow$ & \textbf{ADCC} $\uparrow$ & \textbf{IC} $\uparrow$ & \textbf{ADD} $\uparrow$ & & AD $\downarrow$ & Coh $\uparrow$ & Com $\downarrow$ & \textbf{ADCC} $\uparrow$ & \quad\textbf{IC} $\uparrow$ & \textbf{ADD} $\uparrow$ \\
\hline
RandomCAM & 66.78 & 51.94 & 22.08 & 48.24 & 5.68 & 19.51 & & 67.07 & 52.01 & 21.82 & 48.09 & 3.21 & 17.69 \\
GradCAM++ & 87.02 & 39.49 & 12.36 & 26.37 & 1.70 & 9.75 & & 80.84 & 51.33 & 16.81 & 35.84 & 1.26 & 12.29 \\
XGradCAM & 70.83 & 53.77 & 19.21 & 45.97 & 5.35 & 16.24 & & 72.15 & 52.59 & 18.76 & 44.62 & 3.32 & 13.86 \\
LayerCAM & 90.79 & 53.90 & \textbf{8.75} & 21.72 & 1.06 & 9.30 & & 88.70 & 63.60 & \textbf{8.73} & 26.05 & 1.20 & 8.25 \\
\cdashline{1-14}
GradCAM-E & \underline{49.88} & \textbf{79.29} & 21.46 & \underline{66.23} & \underline{14.91} & \underline{25.18} & & \underline{52.67} & \textbf{81.61} & 19.61 & \underline{65.47} & \underline{8.70} & \textbf{20.89} \\
ShapleyCAM-E & \textbf{47.90} & \underline{79.28} & 23.00 & \textbf{66.97} & \textbf{16.04} & \textbf{26.27} & & \textbf{50.20} & \underline{81.37} & 18.79 & \textbf{67.14} & \textbf{10.86} & \underline{19.74} \\
\cdashline{1-14}
HiResCAM & 89.57 & 57.61 & \underline{9.92} & 24.14 & 1.04 & 10.94 & & 86.57 & 60.99 & \underline{9.77} & 29.43 & 1.42 & 8.98 \\
ShapleyCAM-H & 79.24 & 59.72 & 13.71 & 39.21 & 3.31 & 14.09 & & 76.87 & 63.28 & 13.16 & 42.52 & 2.81 & 12.56 \\
\cdashline{1-14}
GradCAM & 87.08 & 48.83 & 12.53 & 27.44 & 1.58 & 10.15 & & 81.60 & 48.70 & 17.47 & 34.48 & 1.58 & 11.45 \\
ShapleyCAM & 64.81 & 51.65 & 29.36 & 48.44 & 6.16 & 19.20 & & 72.80 & 50.45 & 25.80 & 42.82 & 2.31 & 16.40 \\
\hline
\end{tabular}
}
\end{table*}

\subsection{Quantitative Comparison of CAM Methods Using the Last Convolutional Layer as the Target Layer}
\label{subsec:expriments_extra}

Table \ref{tab:results2_others} presents the results using the last convolutional layer as the target layer for the non-Swin Transformer networks. In non-Swin Transformer networks, ShapleyCAM-H achieves the best or second-best ADCC score, while ShapleyCAM and ShapleyCAM-H perform well in the IC score. ShapleyCAM and GradCAM also demonstrate strong performance in ADD.

Table \ref{tab:results2_swin} presents the results using the first normalization layer (i.e., the LayerNorm layer) of the last transformer block as the target layer, following the suggestion from \cite{jacobgilpytorchcam}, for Swin Transformers. In these cases, compared to using the layer preceding GAP as the target layer (i.e., another LayerNorm layer) in Table \ref{tab:results1_swin}, CAM methods other than ShapleyCAM-E and GradCAM-E fail to outperform RandomCAM in ADCC, IC, and ADD scores. This observation may be attributed to the challenges inherent in using polynomial functions to approximate the highly complex self-attention mechanism. However, as discussed in Section \ref{subsec:exp}, LayerCAM, GradCAM++, GradCAM-E, and ShapleyCAM-E tend to consistently localize the foreground, serving as good localizers rather than reliable explainers. Thus, we maintain skepticism regarding the explainability of these methods, despite some of them demonstrating strong performance in specific metrics. We leave the investigation of this phenomenon to future research.

In summary, ShapleyCAM and ShapleyCAM-H generally outperform other methods in ADCC, IC, and ADD scores for non-Swin Transformer networks. The case for Swin Transformers is complex; it is difficult to determine which CAM method consistently outperforms others, although ShapleyCAM-E and ShapleyCAM frequently yield competitive results. 

Furthermore, as shown in Tables \ref{tab:results1_other}, \ref{tab:results1_swin}, \ref{tab:results2_others} and \ref{tab:results2_swin}, ShapleyCAM, ShapleyCAM-H, and ShapleyCAM-E outperform GradCAM, HiResCAM, and GradCAM-E across the ADCC, IC, and ADD scores in most cases. This outcome substantiates the effectiveness of the second-order approximation of the utility function, highlighting that incorporating both the gradient and Hessian matrix of neural networks typically leads to more accurate and reliable explanations, thereby confirming the validity of the proposed CRG Explainer.


\subsection{Qualitative Comparison of CAM Methods}
\label{subsec:qualitive_expriments}

In our qualitative analysis, we use VGG-16 as the backbone, following \cite{jiang2021layercam}, and apply the last convolutional layer along with the ReST utility function to generate visual explanations. Images are randomly selected from ILSVRC2012, encompassing single objects, multiple objects of the same label, and multiple objects with different labels.

As shown in Fig. \ref{fig:exp}, when presented with images containing a single object (see the first three rows), ShapleyCAM generates the most complete and accurate explanations, highlighting the faces of the \textit{axolotl}, \textit{great grey owl}, and \textit{ostrich}, while all other methods fail to provide complete explanations.

When presented with images containing multiple objects of the same label (see the middle three rows), GradCAM struggles to identify all the objects and highlights incorrect regions. XGradCAM and GradCAM++ produce relatively chaotic explanations, while HiResCAM, LayerCAM, GradCAM-E, ShapleyCAM, ShapleyCAM-H, and ShapleyCAM-E are able to identify all the relevant objects accurately.

When presented with images containing multiple objects with different labels (see the last three rows), methods such as LayerCAM, XGradCAM, GradCAM-E, GradCAM++, and ShapleyCAM-E sometimes highlight irrelevant foreground elements. ScoreCAM occasionally fails to produce meaningful explanations, likely due to the change in the utility function from pre-softmax to ReST. ShapleyCAM, however, generates a complete explanation without highlighting irrelevant regions. 

In summary, ShapleyCAM provides broader and more accurate explanations, effectively highlighting the relevant objects within the images while avoiding the highlighting of irrelevant regions.


\begin{figure*}[htbp]
    \centering
    \includegraphics[width=1.0\textwidth]{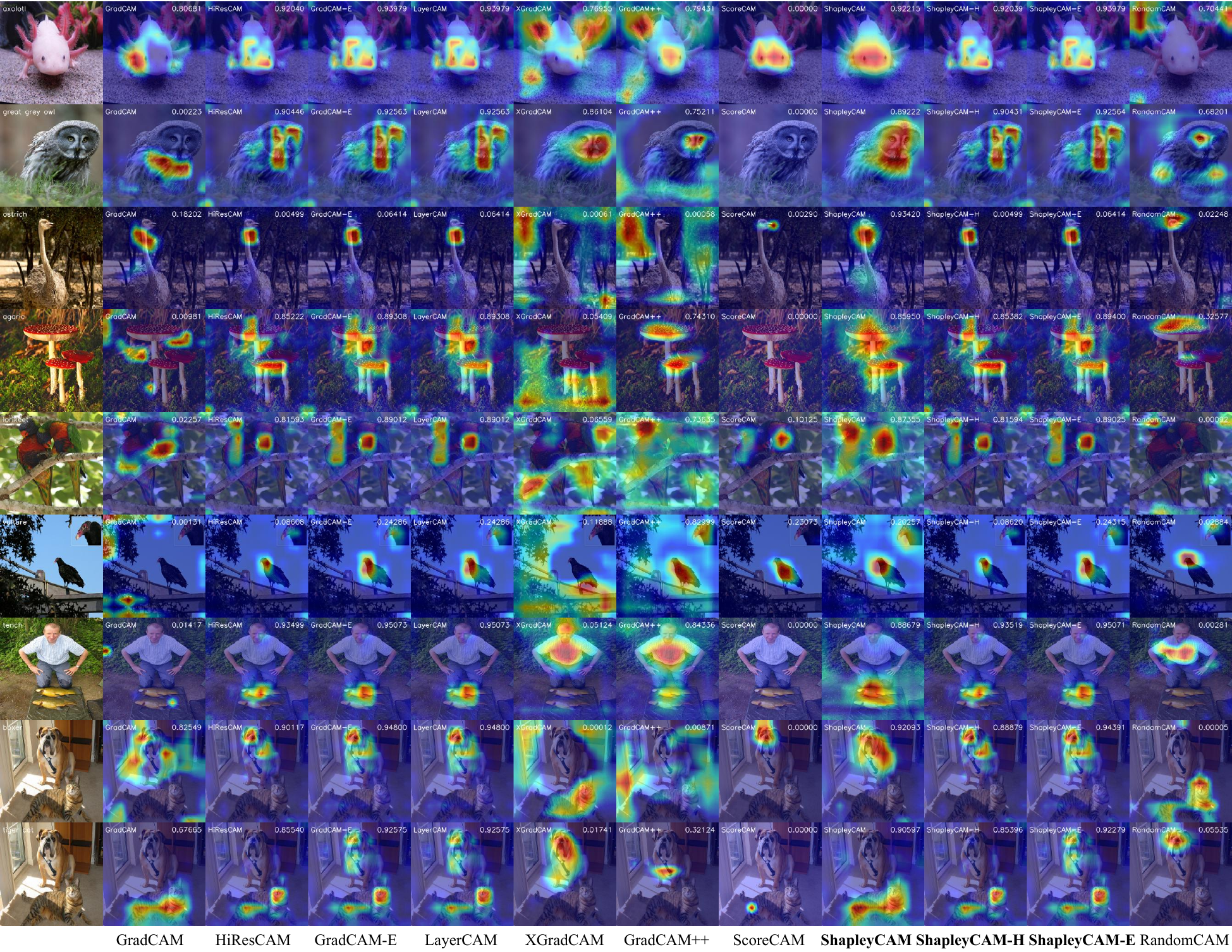} 
\caption{Visual explanation generated by CAMs on VGG-16 using ReST with the last convolutional layer. The labels are in the top left of the first column}
    \label{fig:exp}
\end{figure*}

\begin{figure}[htbp]
    \centering
    \includegraphics[width=0.35\textwidth]{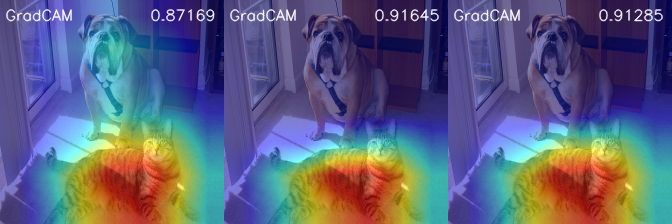} 
    \includegraphics[width=0.35\textwidth]{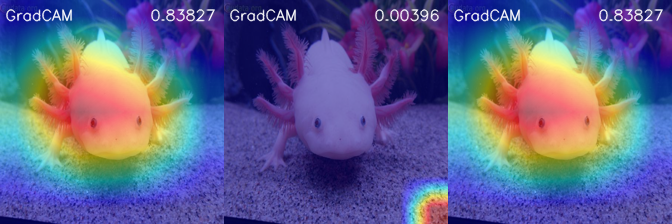} 
    \parbox{0.28\textwidth}{
        \centering
        \hspace{-0.3cm}\footnotesize Pre-softmax \hfill Post-softmax \hfill \textbf{ReST}
    }
    \caption{GradCAM on ResNet-18 using pre-softmax, post-softmax, and ReST with the layer preceding GAP as the target layer. The top row's target class is \textit{tiger cat}, where pre-softmax incorrectly highlights \textit{boxer}. The bottom row shows an image where ResNet-18 predicts \textit{axolotl} with a probability greater than $1 - 10^{-15}$, where post-softmax fails due to gradient vanishing. ADCC scores are shown, with higher values indicating better performance}
    \label{fig:softmax}
\end{figure}

\subsection{Ablation Study of ReST Utility Function}

To verify the effectiveness of the proposed ReST utility function, we conduct an ablation study by comparing the quantitative results of different utility functions on ResNet-18, using GradCAM with the layer preceding the GAP as the target layer and the predicted label as the target class. The qualitative comparison results can be found in Fig. \ref{fig:softmax}. 

As shown in Table \ref{tab:performance_metrics}, ReST outperforms pre-softmax and post-softmax scores, except in Complexity, which also negatively impacts ADCC. The heatmaps from pre-softmax and post-softmax scores may complement each other, as discussed in Section \ref{subsec:softmax}. Since GradCAM with ReST generates a heatmap similar to the combination of pre and post-softmax heatmaps (see Equation \eqref{eq:ReST}), ReST offers improved explainability compared to using only pre-softmax or post-softmax scores, though it may result in more complex explanations.

\begin{table*}[htbp]
\centering
\caption{Performance Comparisons of Different Utility Functions on ResNet-18 with GradCAM}
\label{tab:performance_metrics}
\resizebox{0.5\textwidth}{!}{ 
\begin{tabular}{lcccccc}
    \hline
    \textbf{Method} & AD $\downarrow$ & Coh $\uparrow$ & Com $\downarrow$ & \textbf{ADCC} $\uparrow$ & \textbf{IC} $\uparrow$ & \textbf{ADD} \\
    \hline
    Pre-softmax & \underline{30.20} & 95.93 & \textbf{27.58} & \textbf{77.80} & \underline{23.72} & \underline{42.20} \\
    Post-softmax & 32.55 & \underline{96.23} & \underline{28.87} & 76.38 & 21.41 & 38.89 \\
    ReST         & \textbf{18.71} & \textbf{97.65} & 38.68 & \underline{77.22} & \textbf{33.92} & \textbf{51.90}\\
    \hline
\end{tabular}
}
\end{table*}

\section{Discussion and Future work}
\label{subsec:exp}

The original CAM paper \cite{zhou2016learning} proposes CAM as a localization tool rather than a method for explainability. However, many subsequent works evaluate an explainer's effectiveness based on its localization performance \cite{selvaraju2017grad,jiang2021layercam}, a flawed criterion \cite{hires2020rachel}. Here, we propose to distinguish between \textbf{explainability} and \textbf{localization ability}.

As illustrated in Fig. \ref{fig:main}, methods such as LayerCAM, GradCAM++, GradCAM-E, and ShapleyCAM-E—which apply the ReLU operation before summation—tend to highlight prominent foreground regions regardless of the target class. For instance, when \textit{tiger cat} (or \textit{boxer}) is used as the target class, these methods still highlight regions belonging to \textit{boxer} (or \textit{tiger cat}), even though these areas are irrelevant to the target class. In contrast, methods like GradCAM and ShapleyCAM are more precise, accurately focusing on the regions corresponding to the target class while avoiding other irrelevant foregrounds.

To illustrate this further, we use \textit{yellow lady's slipper} (the least likely class predicted by ResNet-18) as the target class. In this scenario, LayerCAM, GradCAM++, GradCAM-E, and ShapleyCAM-E still emphasize regions corresponding to the true labels (\textit{tiger cat} and \textit{boxer}), revealing a fundamental flaw in their design: these CAM methods are not truly explaining the network's prediction. Instead, they appear to focus on the most obvious foreground objects in the image, regardless of their relevance to the target class. Conversely, other methods like GradCAM demonstrate a more desirable behavior by accurately highlighting the appropriate regions—essentially nothing—which serves as a proper explanation for the neural network's ``absurd" decision-making proccess. 

This observation underscores a critical insight: a good explainer should focus not only on localization but on faithfully reflecting the reasoning behind the model's decision-making process.
Thus, we argue that explainability and localization ability should be distinguished. Methods like LayerCAM, GradCAM++, GradCAM-E, and ShapleyCAM-E sometimes rank highly in quantitative experiments, as in Section \ref{subsec:expriments_extra}.  However, this improvement may stem from their effective exploitation of the localization ability emerging from neural network training for classification task \cite{zhou2016learning}, rather than from a deeper understanding of their internal mechanism. Consequently, using localization-based metrics to assess explainability can be misleading. A true explainer should accurately reflect the model’s inference process and highlight different aspects of the image when using various target classes. From this perspective, visual explanations aligned with human logic remain a more reliable metric for assessing CAMs. We also advocate for further research to develop more precise and comprehensive quantitative metrics for measuring the explainability of CAM methods.



\section{Conclusion}

To enhance the understanding of CAM methods and develop new CAM  methods with improved explainability, this study revisits CAMs from a decision-making perspective. First, we introduce the CRG Explainer to clarify the theoretical foundations of GradCAM and HiResCAM by connecting them to the Shapley value. Then, within the framework, we develop ShapleyCAM, which utilizes gradient and Hessian matrix information for improved heatmap precision. Next, for the choice of the utility function, we analyze the advantages and limitations of pre and post-softmax scores on the explanations, reveal their relationship, and propose the ReST utility function to overcome these limitations. Finally, we validate the effectiveness of ShapleyCAM and its variants through extensive quantitative experiments conducted across 12 network architectures and 6 metrics, utilizing 2 types of target layers. For future work, we emphasize the distinction between explainability and localization ability, calling for further research to establish more precise and comprehensive metrics for evaluating the explainability of CAM methods.

\bibliography{refs}

\end{document}